\documentclass[twoside]{article}

\usepackage[preprint]{aistats2027}
\usepackage[round]{natbib}

\usepackage{graphicx}
\usepackage[utf8]{inputenc} 
\usepackage[T1]{fontenc}    
\usepackage{hyperref}       
\usepackage{url}            
\usepackage{booktabs}       
\usepackage{amsfonts}       
\usepackage{nicefrac}       
\usepackage{microtype}      
\usepackage{xcolor}         
\usepackage{comment}
\usepackage{algorithm}
\usepackage[noend]{algpseudocode}
\usepackage{wrapfig}
\usepackage{bbm}
\usepackage{paralist}
\usepackage{mathtools}
\usepackage{xcolor}

\usepackage{amsmath, amsthm, amssymb} 

\usepackage{subcaption}
\theoremstyle{definition}

\newtheorem{corollary}{Corollary}
\newtheorem{theorem}{Theorem}
\newtheorem{assumption}{Assumption}
\newtheorem{proposition}{Proposition}
\newtheorem{lemma}{Lemma}

\usepackage{amssymb}

\begin{document}

%

%

\twocolumn[

\aistatstitle{Open-ended Scientific Discovery with Possibilistic Reasoning}

\aistatsauthor{
    Anita Yang$^{1,2}$ \And Siu Lun Chau$^{3}$ \And Tomoya Wakayama$^{4}$ \And  Krikamol Muandet$^{5}$ \And Masaki Adachi$^{1}$}

\aistatsaddress{
    \begin{tabular}{c}
    $^{1}$Lattice Lab, Toyota Motor Corporation, Japan \\[-1pt]
    $^{2}$Department of Computer Science, University of Tokyo, Japan \\[-1pt]
    $^{3}$EPIC Lab, College of Computing \& Data Science, Nanyang Technological University, Singapore \\[-1pt] 
    $^{4}$RIKEN Center for Advanced Intelligence Project, Japan \\[-1pt]
    $^{5}$Rational Intelligence Lab, CISPA, Helmholtz Center for Information Security, Germany
    \end{tabular}
    }
]

\fancyhead[CE,CO]{\small\bfseries Open-ended Scientific Discovery with Possibilistic Reasoning}

\begin{abstract}
    Autonomous scientific discovery with LLMs requires generating and testing hypotheses adaptively as evidence accumulates while maintaining statistical validity. Existing anytime-valid methods can handle data-dependent hypotheses, but open-ended discovery poses a deeper challenge: the best discovered hypothesis may still be the best of a bad lot, with better explanations yet undiscovered, while even background knowledge such as physical laws may require revision in light of new findings.
    In response, we formalize the problem as \emph{Abductive Autonomous Scientific Discovery (AASD)} using possibility theory. We introduce abductive utility, a computable measure of discovery progress, and \emph{possibility frontier search}, the first algorithm for AASD, which maintains anytime validity and achieves \(\varepsilon\)-optimal abductive utility asymptotically under suitable conditions. Experiments on synthetic and real-world scientific-discovery tasks show strong performance.
\end{abstract}

\section{Introduction}

Large language models (LLMs; \citealt{achiam2023gpt}) have shown remarkable promise in generating novel scientific hypotheses~\citep{lu2024ai, yamada2025ai, gottweis2026accelerating, ghareeb2026multi}. Their success has been particularly striking in mathematics~\citep{alon2026remarks, openai2026navier, openai2026ten}, where the discovery loop---from hypothesis generation to verification---can be carried out entirely computationally. 
Extending this paradigm to the physical sciences is considerably more challenging, as hypotheses are no longer purely deductive but must be tested through experiments that are noisy, finite, and costly. More fundamentally, mathematics is largely \emph{deductive}~\citep{peirce1878deduction}, allowing formal systems such as Lean~\citep{demoura2015lean} to provide essentially binary verification. Physical science, by contrast, is fundamentally \emph{abductive}: given empirical observations and competing hypotheses, it seeks the explanation that \emph{best accounts} for the evidence~\citep{harman1965inference,lipton2017inference}. Unlike deduction, such inference need not be unique: multiple explanations may remain plausible.

\begin{figure}
    \centering
    \includegraphics[
        width=0.9\linewidth,
    ]{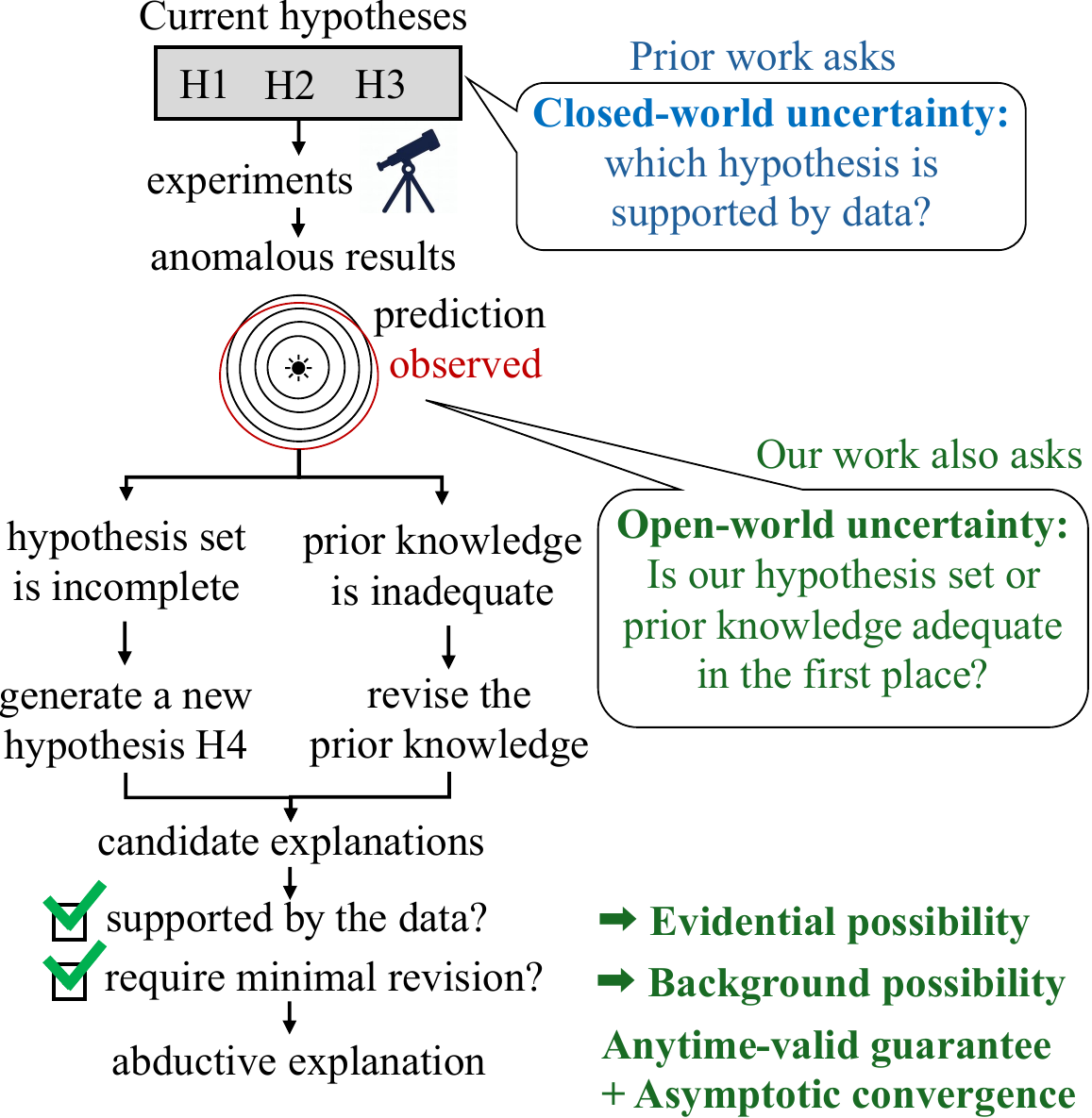}
    \caption{
    Our possibilistic reasoning approach enables open-ended scientific discovery with anytime-valid guarantees under closed-/open-world uncertainty.
    }
    \vspace{-1em}
    \label{fig:concept}
\end{figure}

What, then, makes one plausible hypothesis a \emph{better} explanation than another? Different accounts of abduction emphasise different aspects of scientific reasoning, see \citet{hanson1960logic} and \citet{lipton2017inference}. Here, we specifically draw inspiration from the belief-revision account of abduction developed by \citet{boutilier1995abduction}, particularly its preference for explanations that require less change to existing knowledge. Building on this perspective, we adopt and operationalize three desiderata in a \emph{strict order of priority}. \emph{(1) Empirical adequacy:} a hypothesis must first receive statistical certification from the observed data before it can be selected. \emph{(2) Background knowledge compatibility:} among certified hypotheses, we favour those that can be accommodated with the least change to established scientific knowledge. This preference reflects the empirical scrutiny that such knowledge has survived and is closely related to epistemic entrenchment \citep{dubois1992possibilistic,dubois1991epistemic}. Finally, \emph{(3) Simplicity:} among certified hypotheses tied for the highest background compatibility, we prefer simpler hypotheses. Here, we operationalise simplicity by the number of states a hypothesis permits, commonly motivated by Occam's razor~\citep{rasmussen2000occam,adachi2022fast}.

These criteria guide selection among the hypotheses currently available, but a further difficulty remains: better explanations may not yet have been proposed. This limitation is also known as the \emph{best of a bad lot} in abductive reasoning~\citep{vanfraassen1989laws}. Moreover, the background knowledge guiding this selection may itself be incomplete or inaccurate. Effective scientific discovery must therefore remain open to both new hypotheses and empirically warranted departures from existing knowledge. To this end, we distinguish \emph{closed-world uncertainty}, concerning which hypotheses are supported within a given candidate set, from \emph{open-world uncertainty}, concerning whether that set and the background knowledge used to assess it are themselves adequate (Fig.~\ref{fig:concept}). The challenge is to automate both the assessment of existing hypotheses and the search for alternatives: how can we efficiently generate and test new explanations at scale, and what computable measure can guide the search and quantify progress as new hypotheses and evidence accumulate?

Prior work addresses only parts of this problem. Human-in-the-loop approaches~\citep{lu2024ai, gottweis2025towards} can supply abductive judgment, but do not scale. \citet{huang2025popper,sadhuka2025valuator} use anytime-valid inference~\citep{ramdas2025hypothesis} to safely falsify LLM-generated hypotheses under adaptive experimentation without human intervention. However, falsification alone cannot distinguish among multiple plausible explanations, and these methods assume a fixed hypothesis set. \citet{agarwal2025autodiscovery} instead use Bayesian surprise~\citep{itti2009bayesian} to quantify how strongly observations update an LLM's beliefs. Yet surprise is scientifically meaningful only after plausible known explanations have been excluded.

\textbf{Contributions.}
In response, we first formalize this problem as \emph{Abductive Autonomous Scientific Discovery (AASD)} using \emph{possibility theory}~\citep{dubois1988possibility}. Possibility represents uncertainty by excluding hypotheses inconsistent with evidence. This naturally exposes open-world uncertainty: if all available explanations are excluded, the hypothesis space is incomplete; if viable explanations conflict with background knowledge, that knowledge may require revision. 
We then define \emph{abductive utility}, an automatic measure of monotone discovery progress. It combines data-dependent anytime-valid inference~\citep{martin2026regularized} as \emph{evidential possibility} with guaranteed possibility functions~\citep{dubois2023possibility} as \emph{background possibility}. This provides a computable discovery metric without human intervention. 
Finally, we develop \emph{possibility frontier search}, a principled algorithm for AASD that maintains anytime validity and achieves \(\varepsilon\)-optimal abductive utility asymptotically under suitable conditions. We further derive an observable stopping certificate for \(\varepsilon\)-optimality. We demonstrate strong performance on synthetic and real-world scientific-discovery tasks.

\section{Setup and Background}\label{sec:backgroud}
\vspace{-0.5em}
\subsection{Problem Setting: AASD}
\vspace{-0.5em}
\textbf{Ground-truth.}
In the physical sciences, the ultimate ground truth is nature itself: the underlying environment that generates experimental outcomes. Let \(W_\theta\) denote a model of the environment indexed by \(\theta\in\Theta\), and let \(\mathcal A\) denote the set of available actions a scientist can take (experiments). Performing an action \(A\in\mathcal A\) on an environment \(\theta\) produces a noisy measurement \(M\). We denote its conditional observation law by $M \mid A, \theta \sim \mathbb{P}_\theta(\cdot \mid A)$. Here, \(\mathbb{P}_\theta(\cdot\mid A)\) is the sampling distribution of the measurement implied by the environment model and the observation-noise mechanism.
At each round $t=1,2,\ldots$, we choose an action $A_t$, observe $M_t$, and let $\mathcal D_t=(A_\tau,M_\tau)_{\tau=1}^t$ denote the accumulated experimental record. 
For the theoretical analysis, we assume the model is well specified, so that there exists an unknown \(\theta^\star\in\Theta\) for which the observed measurements are generated according to \(\mathbb{P}_{\theta^\star}(\cdot\mid A_t)\).

\textbf{Abductive target.}
Since the true state $\theta^\star$ is unknown, our goal is to discover a hypothesis $H\subseteq\Theta$ that (i) contains $\theta^\star$, so that it accounts for the underlying phenomenon, and (ii) is compatible with the background knowledge (e.g., physical laws). 
Let $u(H)\in[0,1]$ denote the background-compatibility utility of $H$, defined formally in \S\ref{sec:evidence}, with larger values indicating
less required revision. We define the oracle and operational \emph{abductive utilities}:
\begin{equation}
\begin{aligned}
U^{\star}
&:=
\sup_{H \in \mathcal H_\infty: \; \theta^\star \in H,\; H \neq \Theta} u(H),\\
U_t
&:=
\sup_{H \in \mathsf C_t} u(H), \quad \sup\varnothing:=0.
\label{eq:abductive-possibility}
\end{aligned}
\end{equation}
where $\mathcal H_\infty$ is the generator-reachable, possibly uncountable, hypothesis class and $\mathsf C_t$ is the set of hypotheses certified by round $t$, defined in \S\ref{sec:evidential}.
Thus, $U^\star$ is the utility of the best reachable true explanation, whereas $U_t$ is the utility of the best currently certified one. 

\textbf{Uncertain variables.}
Experimental outcomes \(M_t\) are random variables modeled by \(\mathbb P_\theta(\cdot\mid A)\), whereas \(\theta^\star\) is fixed but unknown ($\neq$ random)---which is referred to as an \emph{uncertain variable}~\citep{houssineau2018parameter,hieu2025decoupling}---and we represent uncertainty about \(\theta^\star\) using a possibility function \(\pi\), as detailed in \S\ref{sec:evidence}.

\textbf{Open-ended hypothesis generation.}
At every finite round $t$, only finitely many hypotheses have been generated and evaluated. Let $\mathcal H_0$ denote the initial candidates and the set $\mathcal H_{\le t}$ denote those initial or generated proper hypotheses at round $t$.

\begin{figure*}[h]
    \centering
    \includegraphics[width=0.7\linewidth]{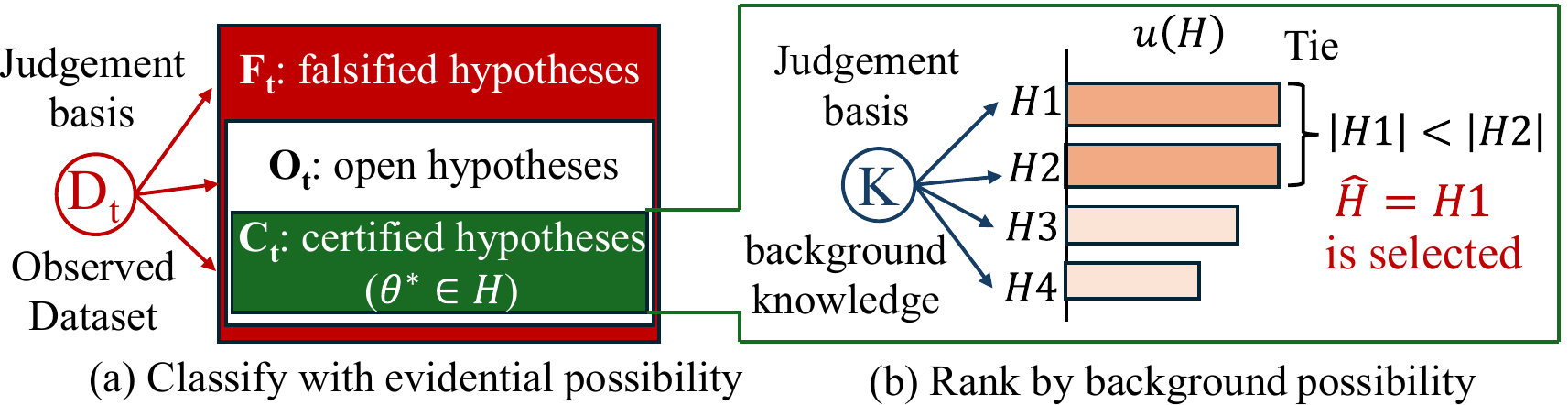}
    \caption{Two-stage assessment. (a) Persistent evidential possibility classifies claims as contradicted, open, or certified. (b) Guaranteed background possibility ranks certified claims by their minimum statewise compatibility. Smaller semantic claims are preferred when scores tie.} 
    \label{fig:possibilities}
    \vspace{-1em}
\end{figure*}
\textbf{Falsifiability.}
We assume that the set of available experiments $\mathcal{A}$ is sufficiently expressive to distinguish between any two distinct parameter values in $\Theta$:
\begin{assumption}[Falsifiability]
\label{ass:falsifiability}
For every $\theta,\vartheta\in\Theta$ such that $\theta\ne\vartheta$, there is $A\in\mathcal A$ with $d_A(\theta,\vartheta)>0$ where $d_A(\theta,\vartheta)$ is a discrepancy between the observation laws \(\mathbb P_\theta(\cdot \mid A)\) and \(\mathbb P_\vartheta(\cdot \mid A)\).
\end{assumption}
For identifiability, Assumption~\ref{ass:falsifiability} only requires this separation to be positive for at least one action whenever \(\theta\neq\vartheta\). A strictly proper scoring rule \citep{gneiting2011making} induces a nonnegative discrepancy of this form. Our finite-sample rate analysis specializes to the logarithmic score, for which \(d_A\) is the directed KL divergence. We treat \(d_A\) as known from the model family.
Falsifiability is a shared assumption in the scientific discovery community~\citep{huang2025popper, agarwal2025autodiscovery}.

\subsection{Background}
\textbf{Possibility theory.}
Possibility theory~\citep{zadeh1978fuzzy,dubois1988applications,dubois1988possibility} provides a representation of uncertainty distinct from probabilistic randomness.
Our information about \(\theta^\star\) is represented by a possibility function \(\pi:\Theta\rightarrow[0,1]\)
where larger \(\pi(\theta)\) indicates that the state \(\theta\) is less strongly excluded by the available information. 
A \(p\)-value is an example \(\pi(\theta)=p(\theta)\): smaller values more strongly exclude \(\theta\), while larger values leave it plausible.
For a hypothesis \(H\subseteq\Theta\), the induced possibility measure is \(\Pi(H)=\sup_{\theta\in H}\pi(\theta)\) and its dual necessity measure is \(N(H)=1-\Pi(H^c)\), where $H^c = \Theta \setminus H$ is the complement.
\(\Pi(H)\) quantifies the extent to which \(H\) remains non-excluded, whereas \(N(H)\) quantifies the extent to which its complement has been excluded. 
A possibility value should therefore not be interpreted as a probability that $H$ is true. In particular, \(\Pi(H)=1\) means only that \(H\) has not been excluded. Under complete ignorance, both \(\Pi(H)=\Pi(H^c)=1\), expressing that neither alternative can be ruled out; by contrast, a probability measure satisfies \(\mathbb P(H)+\mathbb P(H^c)=1\). Consequently, \(\Pi(H)=1\) need not imply \(N(H)=1\): the evidence may simply be insufficient to rule out either $H$ or $H^c$.

\textbf{Anytime-valid inference.}
Anytime-valid inference~\citep{ramdas2023game,ramdas2025hypothesis} provides guarantees robust to optional stopping and adaptive data collection. An \(e\)-value is nonnegative and satisfies \(\mathbb E[E]\le1\) under the null hypothesis $H_\text{null}$, so \(\Pr(E\ge1/\delta)\le\delta\). An \(e\)-process extends this property sequentially: $\mathbb{P}_{H_\text{null}}\left(\sup_t E_t\geq 1/\delta\right)\leq \delta$.
Hence, we may stop when $E_t$ crosses $1/\delta$ without inflating the false-positive rate.
Following~\citet{martin2026regularized}, we define possibility function from $e$-process as \emph{\(e\)-possibility}
\begin{equation}
\label{eq:epos_theta}
\pi_t(\theta):=\min_{0\le s\le t}\{1,E_s(\theta)^{-1}\},
\end{equation}
Then $\pi_t(\theta)\le\delta$ iif
$\max_{0\le s\le t}E_s(\theta)\ge1/\delta$, so the anytime-valid rejection guarantee transfers to $\pi_t$.

\section{Possibilistic Reasoning}\label{sec:evidence}
We formulate AASD using two components of abductive utility: \emph{Evidential possibility} determines the certified set $\mathsf C_t$, using necessity to exclude surviving counterexamples. \emph{Background possibility} supplies the score $u(H)$ used to rank certified claims. 

\subsection{Evidential Possibility}\label{sec:evidential}
As in Fig.~\ref{fig:possibilities}(a), the accumulated evidence classifies hypotheses into three types based on observation $\mathcal D_t$.

\textbf{Version space.}
Fix $\delta\in(0,1)$ and define
\begin{equation}
V_t:=\{\theta\in\Theta:\pi_t(\theta)>\delta\}.
\label{eq:persistent-contour}
\end{equation}
The \emph{version space}~\citep{mitchell1979version} contains the states not yet excluded by the experimental record. Since $\pi_{t+1}(\theta)\le\pi_t(\theta)$, it satisfies $V_{t+1}\subseteq V_t$. We classify candidates as follows (Fig.~\ref{fig:possibilities}(a)):
\begin{align}
\mathsf F_t
&=
\left\{H\in\mathcal H_{\le t}:H\cap V_t=\varnothing\right\}
\label{eq:three-sets}
\\
\mathsf O_t
&=
\left\{
H\in\mathcal H_{\le t}:
H\cap V_t\neq\varnothing,\;
V_t\setminus H\neq\varnothing
\right\}
\notag
\\
\mathsf C_t
&=
\left\{
H\in\mathcal H_{\le t}:
H\cap V_t\neq\varnothing,\;
V_t\setminus H=\varnothing,\;
H\neq\Theta
\right\}
\notag
\end{align}

\textbf{Connection to $e$-possibility.}
Define
\begin{equation*}
\label{eq:persistent-hypothesis-possibility}
\Pi_t(H):=\sup_{\theta\in H}\pi_t(\theta),\qquad
N_t(H):=1-\Pi_t(H^c).
\end{equation*}
We use $\sup\varnothing:=0$. Then $H\cap V_t\ne\varnothing$ iff $\Pi_t(H)>\delta$, while $V_t\subseteq H$ iff $N_t(H)\ge1-\delta$.

\begin{proposition}[\textbf{Anytime-valid version space}]
\label{prop:anytime-valid-version-space}
Let $\mathcal H_{\le t}$ be a set of hypotheses generated by round $t$,
from the same accumulated record $\mathcal D_t$ used to construct $V_t$. For the event $\mathcal E_\delta
:=
\left\{\theta^\star\in V_t, \,\, \forall t\ge0\right\}$, we have
\begin{equation*}
\label{eq:version-space-coverage}
\mathbb P_{\theta^\star}(\mathcal E_\delta)
\ge1-\delta.
\end{equation*}
Moreover, on $\mathcal E_\delta$, simultaneously $\forall t\ge0$ and $\forall H\in\mathcal H_{\le t}$,
\begin{align*}
\theta^\star\in H&\ \Longrightarrow\ H\cap V_t\ne\varnothing,
&\text{(Not contradicted)}\\
V_t\subseteq H&\ \Longrightarrow\ \theta^\star\in H.
&\text{(Certified)}
\end{align*}
In particular,
\[
\mathbb P_{\theta^\star}\!\left(
\exists t,\exists H\in\mathsf C_t:\theta^\star\notin H
\right)\leq \delta.
\]
\end{proposition}
The proof is in Appendix~\ref{apd:simultaneous-hypothesis}. Every false certification implies $\pi_t(\theta^\star)\le\delta$, so all such errors are covered by the same statewise failure event. Thus, this holds \emph{uniformly over time and generated hypotheses}. On $\mathcal E_\delta$, true candidates are never pruned and certified candidates stay certified. Thus $\mathsf C_t\subseteq\mathsf C_{t+1}$ and $U_{t+1}\ge U_t$.

\subsection{Background Possibility}
Fix a background compatibility contour $\pi^K:\Theta\to[0,1]$ throughout discovery. Unlike evidential possibility, the background possibility is based on the fixed background knowledge $K$, which may be elicited from LLMs~\citep{yang2026verbalizing} or historical data (see \S\ref{sec:experiment}).
For a nonempty hypothesis, define background-compatibility utility using \emph{guaranteed possibility}~\citep{dubois2023possibility, dubois2001not}
\begin{align}
u(H):=\inf_{\theta\in H}\pi^K(\theta).
\label{eq:guaranteed-background}
\end{align}
This established set function differs from ordinary possibility: it measures the \emph{minimum} compatibility among the represented states permitted by $H$. The infimum is over every state satisfying the membership predicate of $H$, including states already excluded by experiments\footnote{The infimum is not over $H\cap V_t$: this intersection equals $V_t$ for every $H \in \mathsf C_t$, which assigns them all the same score.}. The numerical contour and the claim's semantics remain fixed, so $u(H)$ does not change over time.

\begin{proposition}[Uniform background accommodation]
\label{thm:minimal-epistemic-change}
Let $H\ne\varnothing$ be measurable. Among arbitrary normalized
measurable revised contours, define
\begin{equation*}
R_K(H):=\inf_{\pi'}\left\{ \|\pi'-\pi^K\|_\infty: \pi'=1\text{ on }H \right\},
\label{eq:min_eps_rev}
\end{equation*}
where the infimum is over normalized measurable contours $\pi'$.
Then $R_K(H)=1-u(H)$. An optimizer is one on $H$ and
unchanged from $\pi^K$ on $H^c$.
\end{proposition}

\begin{proof}
Every admissible revised contour satisfies
\[
\|\pi'-\pi^K\|_\infty \ge \sup_{\theta\in H}(1-\pi^K(\theta)) = 1-u(H).
\]
Raising the contour to one on $H$ and leaving it unchanged
on $H^c$ attains this lower bound.
\end{proof}

This is a candidate-wise counterfactual cost of uniformly accommodating the states allowed by $H$, not a global update of $K$. 
Thus, maximizing \(u(H)\) minimizes this revision cost.

\subsection{Lexicographic preferences and simplicity}
When $\mathsf C_t\neq\varnothing$, we select
\begin{equation}
\widehat H_t\in
\operatorname*{arg\,min}_{H\in\mathsf C_t,\;u(H)=U_t}|H|,
\label{eq:specificity-tiebreak}
\end{equation}
where $|H|$ is the number of states represented by $H$. 
Because \(u(H)\) is antitone under set inclusion, \(H_1\subseteq H_2\) implies \(u(H_1)\ge u(H_2)\); cardinality therefore breaks ties by favoring more specific hypotheses with equal utility.
Eq.~\eqref{eq:specificity-tiebreak} thus induces the lexicographic preference:
$\text{evidential validity}
>_{\mathrm{lex}}
\text{background compatibility}
>_{\mathrm{lex}}
\text{simplicity}$.
If $\mathsf C_t=\varnothing$, we set $\widehat H_t=\bot$. Thus evidence takes precedence, followed by background compatibility and, finally, simplicity as semantic specificity within the chosen finite representation.

\section{Open-ended discovery}
\label{sec:possibility-cut-algorithm}
\begin{algorithm}[t]
\caption{Possibility Frontier Search}
\label{alg:possibility-frontier-search}
\footnotesize
\begin{algorithmic}[1]
\State $\mathcal D_0\gets\varnothing$, $E_0\equiv1$, $\pi_0\equiv1$, $\mathcal H_{\le-1}\gets\mathcal H_0$
\For{\(t=0,\ldots,T\)}
    \State $V_t\gets\{\theta:\pi_t(\theta)>\delta\}$
    \State $\mathcal H_{\le t}\gets$
    \Call{Expand}{$K,\mathcal D_t,V_t,\mathcal H_{\le t-1},B$}
    \label{alg_line:expand}
    \State $(\mathsf C_t,\mathsf O_t, U_t,\overline U_t,\widehat H_t,\mathsf O_t^+) \gets$ Eqs.~(\ref{eq:three-sets},\ref{eq:abductive-possibility},\ref{eq:observable-abductive-gap},\ref{eq:specificity-tiebreak},\ref{eq:active-open-set})
    \If{$\mathsf C_t\ne\varnothing$ and
    $\overline U_t-U_t\le\varepsilon$ and $\mathsf O_t^+=\varnothing$}
        \State \Return $(\widehat H_t,U_t,
        \text{certified }\varepsilon\text{-optimal})$
    \EndIf
    \If{$t=T$}
        \State \Return $(\widehat H_t,U_t,\text{budget exhausted})$
    \EndIf
    \State $\rho_{t+1}\gets$ Eq.~\eqref{eq:cut-directed-design};
       $A_{t+1}\sim\rho_{t+1}$; observe $M_{t+1}$
    \label{alg_line:experiment}
    \State $\mathcal D_{t+1}\gets
           \mathcal D_t\cup\{(A_{t+1},M_{t+1})\}$;
           update $E_{t+1}$
    \State $\pi_{t+1}\gets
           \min\{\pi_t,1,E_{t+1}^{-1}\}$ pointwise
\EndFor
\end{algorithmic}
\end{algorithm}

\begin{algorithm}[t]
\caption{\textsc{Expand}: Hypothesis Expansion}
\label{alg:expand}
\footnotesize
\begin{algorithmic}[1]
\State $\mathcal H_{\le t}\gets
\{H\in\mathcal H_{\le t-1}:H\ne\Theta,\ H\cap V_t\ne\varnothing\}$
\For{$j=1,\ldots,B$}
    \State $\widetilde H_{t,j}\sim\mathsf{Gen}
    (\cdot\mid K,\mathcal D_t,V_t,\mathcal H_{\le t-1})$
    \If{$\widetilde H_{t,j}\ne\Theta$ and
        $\widetilde H_{t,j}\cap V_t\ne\varnothing$}
        \State $\mathcal H_{\le t}\gets
            \mathcal H_{\le t}\cup\{\widetilde H_{t,j}\}$
        \If{$V_t\subseteq\widetilde H_{t,j}$}
        \Comment{Certified}
            \State \textbf{break}
        \EndIf
    \EndIf
\EndFor
\State \Return \(\mathcal H_{\leq t}\)
\end{algorithmic}
\end{algorithm}

We introduce \emph{Possibility Frontier Search} (Alg.~\ref{alg:possibility-frontier-search}), alternating between hypothesis expansion and experiment selection to resolve promising open hypotheses.

\subsection{Hypothesis Expansion}\label{sec:hypothesis_expansion}

Used by Alg.~\ref{alg:possibility-frontier-search}, \textsc{Expand} in  Alg.~\ref{alg:expand} draws at most \(B\) hypotheses and stops early at the first proper certified claim. A proposal that compiles into $H\subseteq\Theta$ is retained when $H\ne\Theta$ and $H\cap V_t\ne\varnothing$, and is certified when $\varnothing\ne V_t\subseteq H$ (recall \S\ref{sec:evidential}).
Alg.~\ref{alg:expand} uses rejection sampling for concreteness; implementation details and the induced conditional proposal law appear in
Appendix~\ref{app:generator-coverage}.

\subsection{Possibility Frontier Search}\label{sec:pos_front}
\begin{figure}
    \centering
    \includegraphics[width=0.95\linewidth]{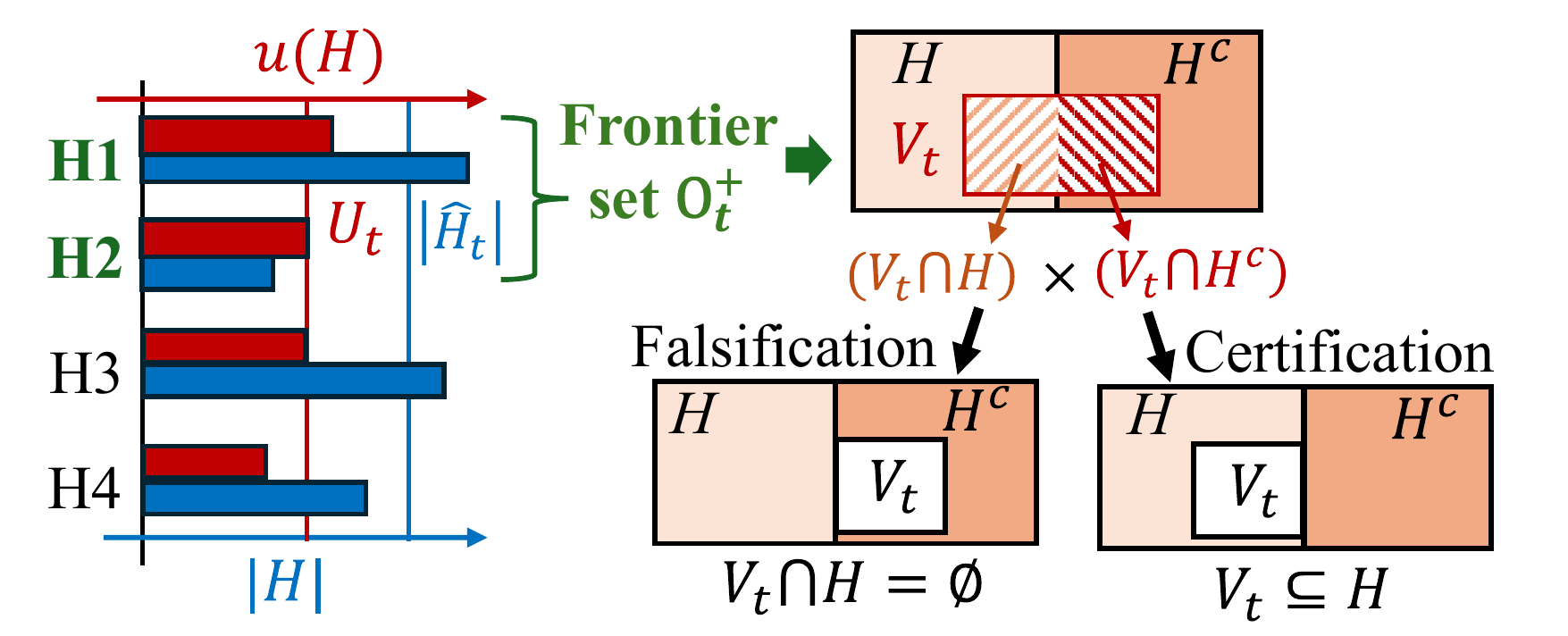}
    \caption{Possibility frontier search collects open hypotheses that can improve the current best \(U_t\) or $|\widehat H_t|$, then selects experiments to falsify or certify them.}
    \label{fig:pfs}
    \vspace{-1em}
\end{figure}
Fig.~\ref{fig:pfs} illustrates the intuition: 
an open candidate can improve the current best lexicographic preferences. If $\mathsf C_t=\varnothing$, set $\mathsf O_t^+=\mathsf O_t$. Otherwise, define
\begin{align}
\mathsf O_t^+=\left\{H\in\mathsf O_t:
\begin{array}{l}
u(H)>U_t,\ \text{or}\\
u(H)=U_t,\ |H|<|\widehat H_t|
\end{array}\right\}.
\label{eq:active-open-set}
\end{align}
The frontier pair set is
\begin{equation*}
\mathcal P_t=
\begin{cases}
\displaystyle\bigcup_{H\in\mathsf O_t^+}
(V_t \cap H) \times (V_t \cap H^c), &\mathsf O_t^+\ne\varnothing,\\[1mm]
\{(\theta,\vartheta)\in V_t^2:\theta\ne\vartheta\},&\text{otherwise}.
\end{cases}
\label{eq:active-possibility-cut}
\end{equation*}
For $\mathcal P_t$, choose
\begin{equation}
\rho_{t+1}\in\operatorname*{arg\,max}_{\rho\in\Delta(\mathcal A)}
\min_{(\theta,\vartheta)\in\mathcal P_t}
\sum_{A\in\mathcal A}\rho(A)d_A(\theta,\vartheta).
\label{eq:cut-directed-design}
\end{equation}
If $\mathcal P_t=\varnothing$, use a fixed reference allocation. The inner minimization identifies the least-separated active pair, while the outer maximization selects the experiment distribution that maximizes its expected discrepancy. For finite pair and action sets, this can be implemented by a linear program; see Appendix~\ref{app:frontier-lp}. 

\textbf{Stopping with $\varepsilon$-optimality.}
Define
\begin{equation}
\overline U_t:=\sup_{\theta\in V_t}\pi^K(\theta).
\label{eq:observable-abductive-gap}
\end{equation}
On $\mathcal E_\delta$, whenever $\mathsf C_t\ne\varnothing$,
$U_t\le U^\star\le\pi^K(\theta^\star)\le\overline U_t$.
For a fixed tolerance $\varepsilon>0$, we can stop at
\begin{equation}
\tau_\varepsilon:=\inf\left\{t\ge0:
\begin{array}{l}
\mathsf C_t\ne\varnothing,\ \mathsf O_t^+=\varnothing,\\
\overline U_t-U_t\le\varepsilon
\end{array}\right\},
\label{eq:abductive-stopping-time}
\end{equation}
with $\inf\varnothing:=\infty$. The safeguard $\mathsf O_t^+=\varnothing$ prevents stopping while a known open candidate $\mathsf O_t$ could improve the lexicographic preference. 
When the stopping condition is reached, the probability of returning a false or \(\varepsilon\)-suboptimal claim is at most \(\delta\); finite termination is not guaranteed.
(Theorem~\ref{thm:observable-abductive-stopping} in Appendix~\ref{app:observable-abductive-stopping}).

\subsection{Theoretical Analysis}
\label{sec:generator-relative-recovery}

\textbf{Reachable goal.}
For $\varepsilon>0$, define
\begin{equation}
\mathcal H_\varepsilon^\star:=
\{H\in\mathcal H_\infty:\theta^\star\in H,
\ u(H)\ge U^\star-\varepsilon\}.
\label{eq:epsilon-optimal-generator-set}
\end{equation}
We call $\varepsilon$-recovery as $U_t\ge U^\star-\varepsilon$. We additionally establish the presence of a certified target, $\mathsf C_t\cap\mathcal H_\varepsilon^\star\ne\varnothing$, which does not count an empty archive as a discovery when the target score is zero. Recovery is analyzed for the generator--policy continuation with termination and the external observation cap disabled $T \to \infty$.

\textbf{Generator coverage.}
Once a near-optimal claim is certifiable, we assume that each budgeted expansion has conditional probability at least $q_\varepsilon>0$ of producing a certified near-optimal claim, unless one is already certified. Appendix~\ref{app:generator-coverage} states this (Assumption~\ref{ass:generator-coverage}). This is a standard reachability condition in stochastic search~\citep{solis1981minimization, adachi2026open}.

\textbf{Rate model.}
For recovery rates, we use the finite-state likelihood-ratio mixture (Eq.~\eqref{app:eq:likelihood-ratio-eprocess})
and conditional sub-Gaussian model (Eq.~\eqref{app:eq:centered-one-step-llr}) in
Appendix~\ref{app:likelihood-ratio-rate}, which is standard modeling choice for $e$-processes~\citep{vovk2021values, grunwald2024safe, howard2020time}.

\begin{theorem}[\textbf{$\epsilon$-recovery guarantee}]
\label{thm:generic-approximate-recovery}
Assume $\Theta$ and $\mathcal A$ are finite and use the full-support likelihood-ratio mixture in Eq.~\eqref{app:eq:likelihood-ratio-eprocess}.
Fix $\varepsilon>0$ and assume a conditional sub-Gaussian proxy $\sigma_\varepsilon^2>0$ as in Eq.~\eqref{app:eq:centered-one-step-llr}. Suppose Assumptions~\ref{ass:falsifiability} and~\ref{ass:generator-coverage} hold, and some $H_\varepsilon\in\mathcal H_\varepsilon^\star$ satisfies Eq.~\eqref{app:eq:frontier-information-rate} with $\Gamma_\varepsilon>0$.
Set $m_\varepsilon:=|\Theta\setminus H_\varepsilon|,\; b_\delta:=\log\frac{1}{\delta w_{\min}}$, where $w_{\min}>0$ is the smallest mixture weight.
For every $t\ge0$,
\begin{equation}
\mathbb P_{\theta^\star}
\bigl(\mathsf C_t\cap\mathcal H_\varepsilon^\star=\varnothing\bigr)
\le\delta+\mathcal B_t(\varepsilon),
\label{eq:optimized-conditional-recovery-bound}
\end{equation}
where for $\tau\ge1$,
\begin{align*}
\mathcal B_t(\varepsilon)
&:=\inf_{0\leq\tau\leq t}
\left\{
\underbrace{R_\tau(\varepsilon)}_\text{not yet certifiable}+\underbrace{(1-q_\varepsilon)^{t-\tau+1}}_\text{generation error}
\right\},\\
R_\tau(\varepsilon)
&:=1\wedge m_\varepsilon
\exp\!\left\{-\frac{(\tau\Gamma_\varepsilon-b_\delta)_+^2}
{2\tau\sigma_\varepsilon^2}\right\}, \; R_0(\varepsilon):=1. 
\end{align*}
The same bound holds for
$\mathbb P_{\theta^\star}(U_t<U^\star-\varepsilon)$.
\end{theorem}

\textbf{Interpretation.}
For all large $t$ with $\tau=\lfloor t/2\rfloor$,
\begin{align*}
\mathcal B_t(\varepsilon)
&\leq
m_\varepsilon
 e^{-\lfloor t/2\rfloor\Gamma_\epsilon^2/(8\sigma_\epsilon^2)}
\nonumber +
e^{-q_\varepsilon(t-\lfloor t/2\rfloor+1)}
\longrightarrow0.
\label{eq:recovery-vanishing-main}
\end{align*}
Hence, a certified near-optimal hypothesis is eventually obtained and retained almost surely on $\mathcal E_\delta$ along the nonterminated continuation (see Appendix~\ref{app:recovery-details}).

\subsection{Related Work}
\textbf{Autonomous scientific discovery.}
Prior work explores hypotheses via search and generation~\citep{gottweis2025towards,jansen2025codescientist,yamada2025ai,ghareeb2026multi,marwitz2026predicting},
ranking them by predefined, human/LLM, or Bayesian-surprise scores~\citep{zhang2024omni,lu2024ai,jansen2025codescientist,itti2009bayesian,agarwal2025autodiscovery}.
These methods do not formulate discovery as abduction.

\textbf{Statistical hypothesis evaluation.}
Prior work evaluates hypotheses with $p$-values~\citep{lin2025llms,movva2025sparse,brightthonney2025autoscidact}, embeds tests in generation~\citep{sengupta2026experigen,zhao2026discoper}, regret-based~\citep{ziomek2024bayesian,ziomek2025time,xu2024principled}, or divergence-based~\citep{fujisawa2025scalable, fujisawa2026generalised} but within predefined models. Sequential $e$-value is applied to LLM-based discovery~\citep{huang2025popper,xu2024online} likewise do not address open-ended adaptive generation. $e$-possibility~\citep{martin2026regularized} covers data-dependent claims and ours is the first to integrate it with growing hypotheses, fixed background preference, and frontier-based experimental design.

\textbf{Online hypothesis testing.}
Our design follows Chernoff-style active testing~\citep{chernoff1959sequential,naghshvar2013active,garivier2016optimal,mukherjee2022chernoff},
but gives maximin optimality on a data-dependent possibility frontier rather than asymptotic fixed-family optimality (Appendix~\ref{app:scope-optimality}).

\section{Synthetic Experiment}
\label{sec:experiment}

We first evaluate our framework in a controlled synthetic setting where the ground-truth state $\theta^\star$ is known. This allows us to test the two central challenges of open-ended discovery in isolation: the initial hypothesis set is incomplete, so discovering $\theta^\star$ requires hypothesis expansion, while the background knowledge is informative but misspecified, so successful abduction must balance empirical validity against departure from prior knowledge. The construction mirrors the finite-state setting of our theoretical analysis; full experimental details are provided in Appendix~\ref{app:synthetic}.

\subsection{Environmental Setup}
\begin{figure}
    \centering
    \includegraphics[width=0.8\linewidth]{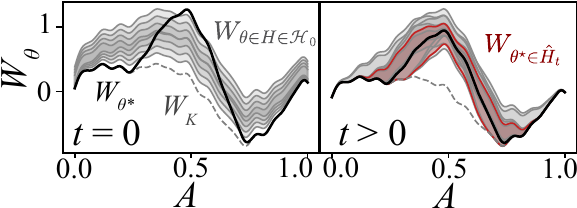}
    \caption{Synthetic discovery setup:
    The environment $W_{\theta^\star}$ differs from the background model $W_K$ (dashed). At $t=0$ (left), none of the initial hypotheses    $H\in\mathcal H_0$ contains $\theta^\star$. For $t>0$ (right), the hypothesis space is expanded to discover a hypothesis containing $\theta^\star$.
    }
    \vspace{-1em}
    \label{fig:setup}
\end{figure}
\textbf{Ground-truth.}
Consider an environment $W_{\theta}$:
\[
W_\theta(A)=W_K(A)+
\frac a2\left(1+\cos\frac{\pi(A-A_0)}{w}\right)
\mathbbm 1_{\{|A-A_0|\le w\}},
\]
where $\theta=(a,w)$ and the ground-truth parameter is $\theta^\star=(1.00,0.28)$. The action space $\mathcal A$ consists of 73 equispaced locations. At round $t$, selecting $A_t\in\mathcal A$ produces
$M_t=W_{\theta^\star}(A_t)+\epsilon_t,$ where $\epsilon_t\sim\mathcal N(0,\sigma_E^2).$
The learner does not observe $\theta^\star$ and must identify a hypothesis containing it from sequential experiments.

\textbf{Hypothesis generation.}
We discretize $\Theta$ into 1,540 parameter states and construct 44 reachable hypotheses $\mathcal H_\infty$, of which three contain $\theta^\star$. The initial set $\mathcal H_0$ is deliberately misspecified: no $H\in\mathcal H_0$ contains $\theta^\star$. Consequently, a closed-world algorithm within $\mathcal H_0$ cannot recover the ground truth. Our method expands the candidate set according to Alg.~\ref{alg:expand}, with a proposal budget of $B=4$ per round, and seeks an abductively preferred certified hypothesis containing $\theta^\star$.

\textbf{Evidential possibility.}
From the accumulated observations $\mathcal D_t=\{(A_\tau,M_\tau)\}_{\tau=1}^t$, we construct a Gaussian-mixture likelihood-ratio e-process $E_t(\theta)$~\citep{howard2021time}. This induces the evidential possibility in Eq.~\eqref{eq:epos_theta} and the persistent version space $V_t^\delta$ in Eq.~\eqref{eq:persistent-contour}. We use $\delta=0.1$ to classify hypotheses as in \S\ref{sec:evidential}.

\textbf{Background possibility.}
Background knowledge $K$ is obtained from historical observations $\mathcal D_K=\{(A_i^K,M_i^K)\}_{i=1}^N$ and represented by the Gaussian process (GP: \citealt{rasmussen2004gaussian})
$W^K\mid\mathcal D_K\sim\mathcal{GP}(W_K,C_K)$.
Crucially, this background model is informative but misspecified: $W_K\neq W_{\theta^\star}$ (Fig.~\ref{fig:setup}). This mimics, for example,
a setting in which historical measurements remain useful but the current environment has shifted. Using Gaussian possibility \citep{ristic2020target,denoeux2014likelihood}, we construct the background possibility $\pi^K$ and compute the utility $u(H)$ in Eq.~\eqref{eq:guaranteed-background}.

\textbf{Experimental design.}
We instantiate $d_A(\theta,\vartheta)$ using the KL divergence between the Gaussian observations. At each round, the next experiment is selected using the possibility-frontier maximin design in Eq.~\eqref{eq:cut-directed-design}, and the complete
procedure follows Alg.~\ref{alg:possibility-frontier-search}.

\begin{figure}
    \centering
    \includegraphics[width=0.8\linewidth]{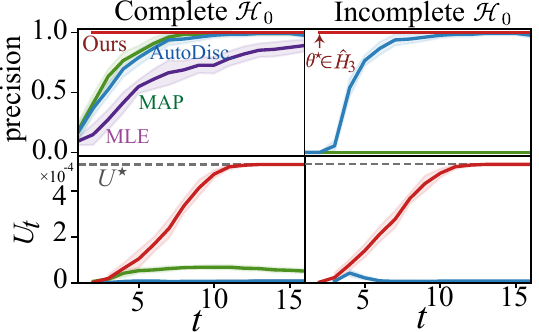}
    \caption{
    Precision (upper) and abductive utility (lower) under complete (left) incomplete $\mathcal H_0$. Ours returns a certified hypothesis at $t=2$ or $3$ with empirical precision $1.0$, and its abductive utility subsequently improves and converges to $U^\star$. Shaded regions denote 95\% confidence intervals.}
    \label{fig:exp2_precision_abductive_pos}
    \vspace{-1em}
\end{figure}
\begin{figure}[b]
    \centering
    \includegraphics[width=0.95\linewidth]{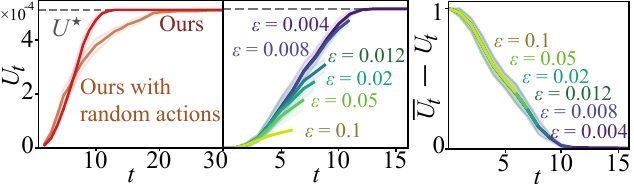}
    \caption{Ablation study: (a) experimental design and (b) stopping condition for $U_t$ and (c) $\overline{U}_t - U_t$. Ours reaches $U^\star$ faster than random actions, and larger tolerance $\varepsilon$ stops earlier following the gap.}
    \label{fig:exp2_ablation}
\end{figure}
\subsection{Experiments}
\textbf{Baselines.}
We compare Possibility Frontier Search against MLE, MAP, and AutoDiscovery~\citep{agarwal2025autodiscovery}.
MLE and MAP select hypotheses using, respectively, Gaussian likelihood and posterior probability under the background GP prior. Both are restricted to the initial hypothesis set $\mathcal H_0$, and thus represent closed-world inference when $\mathcal H_0$ is incomplete. AutoDiscovery, in contrast, adaptively expands the hypothesis space using Monte Carlo tree search~\citep{silver2016mastering} and selects hypotheses according to Bayesian surprise. Because these baselines do not prescribe an adaptive experimental-design policy over $\mathcal A$, we use uniform random sampling, $A_t\sim\mathcal U(\mathcal A)$, for their experiments. Our method instead selects experiments using the possibility-frontier design in Eq.~\eqref{eq:cut-directed-design}.

\textbf{Evaluation.}
We evaluate (i) \emph{precision}, indicating whether a returned hypothesis contains or not $\theta^\star \in \widehat{H}_t$, and (ii) the abductive utility $U_t$ in Eq.~\eqref{eq:abductive-possibility}. Results are averaged over eight random seeds; hence the reported precision may take intermediate values across runs.

\textbf{Effect of incomplete $\mathcal H_0$.}
We compare initial archives containing at least one true claim, $\exists H\in\mathcal H_0:\theta^\star\in H$, with archives containing no true claim in Fig.~\ref{fig:exp2_precision_abductive_pos}. Under complete $\mathcal H_0$, all methods eventually reach $\theta^\star \in \widehat{H}_t$, while ours does so already at $t=2$. In contrast, under incomplete $\mathcal H_0$, only AutoDiscovery and ours can reach a true hypothesis, as vanilla MLE and MAP do not handle an expanding hypothesis space. On abductive utility $U_t$, ours converges to $U^\star$ while others cannot, illustrating empirical recovery behavior. Fig.~\ref{fig:exp2_example} illustrates the corresponding search trajectories. AutoDiscovery can find a true hypothesis, but its surprise objective favors hypotheses far from the prior, preventing convergence to $U^\star$.

\textbf{Ablation study.} 
Fig.~6(a) compares frontier allocation (Eq.~\eqref{eq:cut-directed-design}) with random actions, $A_t\sim\mathrm{Unif}(\mathcal A)$. The frontier policy finds $U^\star$ sooner, demonstrating the benefit of adaptive experimental design. 
Panels (b)(c) illustrate how larger tolerances relax the observable gap threshold and trigger earlier stopping in these runs. 

\begin{figure}
    \centering
    \includegraphics[width=0.7\linewidth]{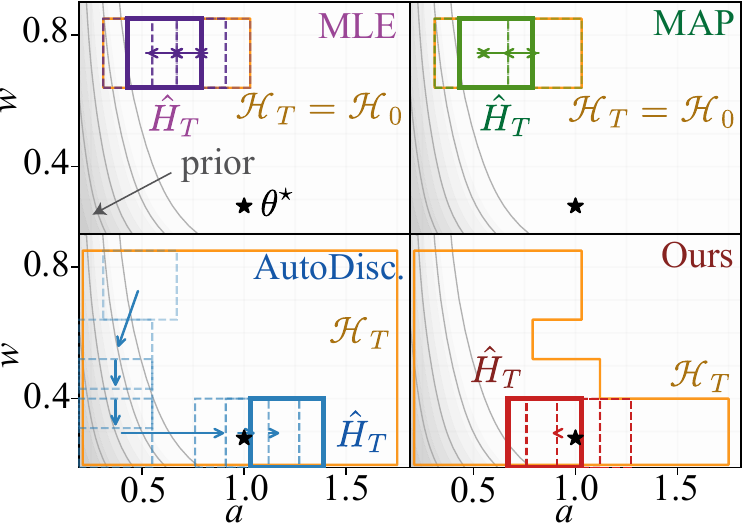}
    \caption{
MLE and MAP remain confined to the initial set $\mathcal H_0$. AutoDiscovery expands the search toward hypotheses that are most surprising relative to the prior, whereas our method rapidly identifies a certified hypothesis $H\in\mathsf C_t$ requiring minimal revision of the background knowledge. Orange box denotes $\mathcal H_{\leq T}$.}
    \label{fig:exp2_example}
    \vspace{-1em}
\end{figure}

\section{Real-world dataset}
Next, we apply our framework to natural-language hypothesis discovery on real-world datasets. These real-data experiments are exploratory case studies using approximate statistical tests and LLM-based generation, for which exact anytime validity and recovery guarantees are not established. Full experimental details are provided in Appendix~\ref{app:natural_language_discovery}.

\subsection{Environmental Setup}\label{sec:real_env}
\textbf{Datasets as environments.}
Following DiscoveryBench~\citep{majumder2025discoverybench}, two real-world datasets \emph{Incarceration}/\emph{SES} are used as empirical environments.

\textbf{Research question and gold hypothesis.}
Unlike \S\ref{sec:experiment}, the environmental state is unknown. DiscoveryBench instead provides a research question to the learner and a held-out \emph{gold hypothesis} for evaluation, e.g.,
``How did the wealth levels of individuals with a history of incarceration compare to those never incarcerated in 1996?''
The gold hypothesis states that previously incarcerated individuals have lower wealth than those never incarcerated.
Each dataset starts with five natural-language hypotheses, each paired with its logical complement.

\textbf{Parameterization.}
We represent the experimentally accessible environment by population quantities $\theta=(\theta_1,\ldots,\theta_J)\in\Theta$.
An action $A\in\mathcal A$ is a \emph{measurement query} returning a noisy measurement of $\theta_{j_A}$. For example, let $\theta_j=\mu_{\mathrm{inc}}-\mu_{\mathrm{never}}$, the difference in mean wealth between individuals with and without an incarceration history. Then $H=\{\theta:\theta_j<0\}$, and the query returns $M=\widehat{\mu}_{\mathrm{inc}}-\widehat{\mu}_{\mathrm{never}}$ on unseen data.
Because \(\mathcal A\) is fixed before discovery, \(J\) and \(\Theta_{\mathrm{grid}}\) remain fixed.

\textbf{Experimental budget.}
We randomly partition each dataset into 16 disjoint folds. At each round, we select up to two measurement queries and fix their statistical tests before observing the next unseen fold, on which the queries are then evaluated. This yields at most 32 test executions over 16 rounds, enabling sequential evidence accumulation while limiting overfitting.

\textbf{Evidential possibility.}
Following POPPER~\citep{huang2025popper}, each measurement query has a statistical test fixed before its data fold is observed. Its $p$-value is converted to an $e$-value. Because up to two tests may share a fold, we average their statewise contributions within each fold and multiply these averages sequentially across folds. The resulting $e$-process defines the evidential possibility and necessity in Eq.~\eqref{eq:epos_theta}.

\textbf{Background possibility.}
We elicit background possibility from an LLM for each estimand, following \citet[Algorithm~2]{yang2026verbalizing}. The fixed joint-state possibility contour is then defined by $\pi^K(\theta)=\min_{j=1,\ldots,J}\pi_j^K(\theta_j)$. Because \(u(H)\) minimizes over allowed states, unspecified estimands can lower the score.

\textbf{Experimental design.}
At each round, POPPER generates candidate measurement queries for hypotheses in $\mathcal H_{\le t}$, excluding unsupported or duplicate queries.
For each query $A$, let $j_A$ denote its target estimand. We model its measurement as $\mathcal N(\theta_{j_A},\sigma_A^2)$, where $\sigma_A$ is the anticipated standard error, and apply the maximin design in Eq.~\eqref{eq:cut-directed-design} to select among the candidates.

\subsection{Experiments}
We use \texttt{GPT-5}~\citep{singh2025openai} for natural-language hypothesis generation and LLM-based evaluation.

\textbf{Baselines.}
We compare against AutoDiscovery~\citep{agarwal2025autodiscovery} under the same datasets, initial conditions, and experimental budget. We omit MLE and MAP because they operate over a prespecified statistical model class and do not directly provide open-ended natural-language hypothesis evaluation in this setting.

\textbf{Metrics.}
We evaluate hypothesis selection using three metrics:
(i) Hypothesis Matching Score (HMS)~
\citep{bragg2026astabench}, which measures agreement between the selected and gold hypotheses by an LLM;
(ii) abductive utility $U_t$ in Eq.~\eqref{eq:abductive-possibility}; and (iii) cardinality $|\widehat H_t|$ in Eq.~\eqref{eq:specificity-tiebreak}.
For AutoDiscovery, we retrospectively evaluate $U_t$ using the same accumulated evidence.

\textbf{Results.}
Fig.~\ref{fig:discoverybench} shows that our method eventually achieves higher HMS and $U_t$, and lower $|\hat H_t|$ on both datasets. Notably, $U_t$ and $|\hat H_t|$ closely track HMS, suggesting that our utility can monitor discovery progress without a gold hypothesis. This makes it promising for real-world scientific discovery, where no ground-truth hypothesis is available.

\begin{figure}
    \centering
    \includegraphics[width=\linewidth]{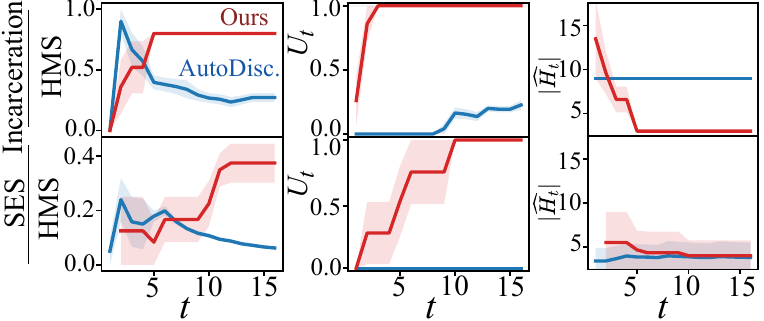}
    \caption{
    Natural-language discovery results. Left: HMS. Middle: abductive utility. Right: cardinality $|\widehat H_t|$. Shaded regions indicate $\pm$ standard error.}
    \label{fig:discoverybench}
    \vspace{-1em}
\end{figure}

\section{Conclusion and limitation}
We introduced AASD as a new problem setting, abductive utility as a discovery metric, and possibility frontier search as its solving algorithm. 
Our notion of open-endedness concerns expansion of candidate hypotheses within a fixed environmental model, action space, and background contour, rather than online expansion of \(\Theta\) or revision of \(K\). Moreover, abductive selection depends on the quality and calibration of the background possibility model, so misspecification may bias which admissible hypotheses are preferred.

\subsubsection*{Acknowledgments}
This work was supported by JST, ARiSE, Japan Grant Number JPMJAR2628.
TW was partially supported by JSPS KAKENHI (26K21188) and RIKEN Incentive Research Project.

\bibliographystyle{plainnat_max1}
\bibliography{references}

\clearpage
\appendix
\thispagestyle{empty}

\onecolumn
\aistatstitle{Supplements}

\section{Hypothesis Representation and Adaptive Validity}
\subsection{Hypotheses as subsets of the state space}
\label{app:hypothesis-representation}
A claim is identified with the states under which its experimentally
testable implications hold:
\begin{equation}
H:=\{\theta\in\Theta:W_\theta\text{ satisfies the claim's testable implications}\}.
\label{app:eq:hypothesis-representation}
\end{equation}
Semantically identical claims have the same membership predicate and
are deduplicated. A natural-language claim must be compiled into this
predicate before certification. A certificate covers that predicate,
not additional causal or mechanistic wording absent from the statistical
representation; see Appendix~\ref{app:implementation}.

\subsection{Adaptive-hypothesis validity}
\label{apd:simultaneous-hypothesis}
Use the persistent contour $\pi_t$ from Eq.~\eqref{eq:epos_theta}.
For every $\theta$ and every $\gamma\in(0,1)$,
\begin{equation}
\mathbb P_\theta(\exists t:\pi_t(\theta)\le\gamma)
=\mathbb P_\theta(\exists t:E_t(\theta)\ge1/\gamma)
\le\gamma.
\label{app:eq:statewise-validity}
\end{equation}
All events and optimizations are assumed measurable. The statement holds
under the actual adaptive data-collection protocol used to construct the
statewise evidence.

\begin{corollary}[Generic adaptive-hypothesis validity]
\label{cor:simultaneous-hypothesis}
Let the candidate archives be data dependent. For local thresholds
$0<\alpha\le\beta<1$, define
\begin{equation}
\mathsf C_t(\alpha,\beta):=
\{H\in\mathcal H_{\le t}:\Pi_t(H)>\alpha,\ N_t(H)\ge1-\beta\}.
\label{eq:evidence-admissibility}
\end{equation}
Then
\begin{align}
\mathbb P_{\theta^\star}
(\exists t, H\in\mathcal H_{\le t}:\theta^\star\in H,
\Pi_t(H)\le\alpha)&\le\alpha,
\label{app:eq:true-hypothesis-error}\\
\mathbb P_{\theta^\star}
(\exists t, H\in\mathcal H_{\le t}:\theta^\star\notin H,
N_t(H)\ge1-\beta)&\le\beta.
\label{app:eq:false-certification-error}
\end{align}
In particular, with probability at least $1-\beta$, every member of every
$\mathsf C_t(\alpha,\beta)$ is true.
\end{corollary}
\begin{proof}
If $\theta^\star\in H$ and $\Pi_t(H)\le\alpha$, then
$\pi_t(\theta^\star)\le\alpha$. If $\theta^\star\notin H$ and
$N_t(H)\ge1-\beta$, then
\[
\pi_t(\theta^\star)\le\Pi_t(H^c)=1-N_t(H)\le\beta.
\]
Each error reduces to its respective statewise event in
Eq.~\eqref{app:eq:statewise-validity}. No union bound over claims is used.
\end{proof}

\paragraph{Equal-threshold version space.}
With the main-text threshold $\delta$,
\begin{equation}
\mathcal E_\delta=\{\theta^\star\in V_t\text{ for every }t\ge0\},
\qquad
\mathbb P_{\theta^\star}(\mathcal E_\delta)\ge1-\delta.
\label{app:eq:coverage-event}
\end{equation}
In particular,
\[
\mathbb P_{\theta^\star}(\exists t:V_t=\varnothing)\le\delta.
\]
The identities
\[
\Pi_t(H)>\delta\ \Longleftrightarrow\ H\cap V_t\ne\varnothing,
\qquad
N_t(H)\ge1-\delta\ \Longleftrightarrow\ V_t\subseteq H
\]
follow from the definitions of supremum and of $V_t$. Together with
Eq.~\eqref{app:eq:coverage-event}, they prove
Proposition~\ref{prop:anytime-valid-version-space}.
On $\mathcal E_\delta$, a true retained claim cannot be pruned and a
certified claim remains certified at later rounds. Therefore
$\mathsf C_t\subseteq\mathsf C_{t+1}$ and $U_t$ is nondecreasing on that
event. These assertions apply to the procedure's continuation when its
stopping condition is disabled.

\section{Guaranteed Background Possibility and Refinement}
\label{apd:minimal-epistemic-change}
\paragraph{Uniform accommodation.}
The proof of Proposition~\ref{thm:minimal-epistemic-change} appears in the
main text. Its distance is between statewise contours. The optimization
allows arbitrary normalized measurable revised contours and imposes no
necessity condition on the complement of the selected claim. It therefore
measures uniform accommodation, rather than the former binary-assessment
adoption cost. The background contour used by discovery itself stays fixed.

\begin{proposition}[Preference against strict weakening]
\label{prop:strict-refinement}
For nonempty $H_1\subseteq H_2$, $u(H_1)\ge u(H_2)$.
If $\Theta$ is finite and $H_1\subsetneq H_2$, then
\[
(u(H_1),-|H_1|)>_{\mathrm{lex}}(u(H_2),-|H_2|).
\]
Consequently, the reporting rule cannot select a certified hypothesis
when an available certified strict refinement exists.
\end{proposition}
\begin{proof}
The infimum over a subset is at least the infimum over the containing
set. If the inequality in utility is strict, the first coordinate decides
selection. Otherwise the utilities tie and finite strict inclusion gives
$|H_1|<|H_2|$, so the second coordinate decides selection.
\end{proof}
This is an archive-relative preference, not a guarantee that every useful
refinement has already been generated or certified. It requires neither
an antichain nor disjoint candidates. It does not claim that cardinality
is a complete measure of scientific value.

\section{Conditional Hypothesis Generation}
\label{app:generator-coverage}
Let $\mathcal I_t$ denote the history after the observations through round
$t$ but immediately before expansion. Let $\mathcal I_t^+$ include that
round's expansion and scoring decisions, before choosing experiment
$t+1$. Thus $\mathcal I_t\subseteq\mathcal I_t^+\subseteq\mathcal I_{t+1}$.
Let $Q_t(\cdot\mid\mathcal I_t)$ denote the raw proposal kernel.

Define the set of all certifiable proper claims in the generator's
language, not just those already in the archive:
\begin{equation}
\mathcal R_t^{\mathrm{cert}}:=
\{H:\varnothing\ne V_t\subseteq H,\ H\ne\Theta\},
\qquad
p_t^{\mathrm{cert}}:=Q_t(\mathcal R_t^{\mathrm{cert}}\mid\mathcal I_t).
\label{app:eq:feasible-generation-set}
\end{equation}
For $p_t^{\mathrm{cert}}>0$, the distribution of the first certified draw
from a fixed within-round kernel is
\begin{equation}
Q_t^{\mathrm{cert}}(D\mid\mathcal I_t)
:=\frac{Q_t(D\cap\mathcal R_t^{\mathrm{cert}}\mid\mathcal I_t)}
{p_t^{\mathrm{cert}}}.
\label{app:eq:conditional-proposal-kernel}
\end{equation}
Here $D$ is a measurable collection of claims, not a data record.

\begin{proposition}[Rejection-conditioned expansion]
\label{prop:rejection-conditioned-expansion}
Suppose $\widetilde H_{t,1},\widetilde H_{t,2},\ldots$ are conditionally
i.i.d. from $Q_t(\cdot\mid\mathcal I_t)$, and set
$J_t:=\inf\{j\ge1:\widetilde H_{t,j}\in\mathcal R_t^{\mathrm{cert}}\}$.
If $p_t^{\mathrm{cert}}>0$, then $J_t<\infty$ almost surely and
$\widetilde H_{t,J_t}$ has law $Q_t^{\mathrm{cert}}$. Moreover,
\begin{equation}
\mathbb P(J_t>b\mid\mathcal I_t)=(1-p_t^{\mathrm{cert}})^b,
\qquad
\mathbb E[J_t\mid\mathcal I_t]=(p_t^{\mathrm{cert}})^{-1}.
\label{app:eq:rejection-complexity}
\end{equation}
\end{proposition}
\begin{proof}
Conditionally on the fixed history, the acceptance indicators are i.i.d.
Bernoulli variables. The waiting time is geometric, and the accepted draw
has the kernel conditioned on acceptance.
\end{proof}
The operational algorithm caps the attempts at $B_t$. Uncertified proper
proposals meeting $H\cap V_t\ne\varnothing$ are retained as open claims.
On $\mathcal E_\delta$, no true candidate is discarded as contradicted.

\subsection{Formal generator coverage}
For a fixed accuracy $\varepsilon$, define the readiness event
\begin{equation}
W_{t,\varepsilon}:=
\{\mathcal R_t^{\mathrm{cert}}\cap\mathcal H_\varepsilon^\star
\ne\varnothing\}.
\label{app:eq:readiness-event}
\end{equation}
Let $\mathsf C_t^-$ be the previous archive reclassified with $V_t$
before expansion, and let $Z_{t,\varepsilon}=1$ if the actual budgeted
expansion produces a certified member of $\mathcal H_\varepsilon^\star$.

\begin{assumption}[Effective budgeted generator coverage]
\label{ass:generator-coverage}
For the fixed $\varepsilon$, there is $q_\varepsilon\in(0,1]$ such that
\[
\mathbb P(Z_{t,\varepsilon}=1\mid\mathcal I_t)\ge q_\varepsilon
\]
almost surely on
$W_{t,\varepsilon}\cap
\{\mathsf C_t^-\cap\mathcal H_\varepsilon^\star=\varnothing\}$,
for every expansion round of the continuation.
\end{assumption}
This condition need not hold for an unrestricted free-form generator.
It is an explicit assumption on the implemented expansion step.

For conditionally i.i.d. proposals stopped at the first certified draw,
let
\[
q_{t,\varepsilon}^{\mathrm{cond}}
:=\frac{Q_t(\mathcal R_t^{\mathrm{cert}}\cap
\mathcal H_\varepsilon^\star\mid\mathcal I_t)}{p_t^{\mathrm{cert}}}
\quad\text{when }p_t^{\mathrm{cert}}>0.
\]
Then
\[
\mathbb P(Z_{t,\varepsilon}=1\mid\mathcal I_t)
=\left[1-(1-p_t^{\mathrm{cert}})^{B_t}\right]
q_{t,\varepsilon}^{\mathrm{cond}}.
\]
If $p_t^{\mathrm{cert}}=0$, the hit probability is zero.
The assumption lower-bounds the entire expression, not just its second
factor. Every raw attempt, including invalid proposals, duplicates, and
any enumeration calls, counts toward $B_t$. Deterministic enumeration
or changing the kernel within the round requires analysis of that actual
procedure; the i.i.d. equality must not be applied to it automatically.

\subsection{Implementation}
Algorithm~\ref{alg:expand} uses rejection sampling for concreteness. It returns the retained archive after at most $B$ attempts, whether or not a new certified claim was found. Language-model proposal mechanisms may be used within the registered semantic interface. 
Such implementation can be implemented by rejection sampling~\citep{novikov2025alphaevolve, lange2026shinkaevolve, adachi2026regret}.
Their implementation does not by itself establish Assumption~\ref{ass:generator-coverage}.

\section{Possibility Frontier Search as a Fractional Design}
\label{app:frontier-lp}
Suppose $\mathcal A=\{A_1,\ldots,A_m\}$ and the pair set is
$\mathcal P=\{(\theta_i,\vartheta_i)\}_{i=1}^n$. Set
$D_{ij}=d_{A_j}(\theta_i,\vartheta_i)$. The maximin problem
\[
\max_{\rho\in\Delta(\mathcal A)}
\min_i\sum_j\rho(A_j)D_{ij}
\]
is equivalent to the LP in Eq.~\eqref{eq:lp}. Its output is a probability vector
$\rho^\star=(\rho_1^\star,\ldots,\rho_m^\star)$ and an attained separation value $\Gamma^\star$. If $A\sim\rho^\star$, then for every constrained pair
\[
\mathbb E[d_A(\theta_i,\vartheta_i)]
=
\sum_j\rho_j^\star D_{ij}
\geq\Gamma^\star.
\]
This is why random sampling implements the fractional design. A deterministic alternative is to track cumulative desired allocations: with
$N_t(A)=\sum_{s=1}^t\mathbf 1\{A_s=A\}$ and
$W_{t+1}(A)=\sum_{s=0}^t\rho_{s+1}(A)$, select
\begin{equation}
\label{app:eq:deterministic-tracking}
A_{t+1}\in
\operatorname*{arg\,max}_{A\in\mathcal A}
\{W_{t+1}(A)-N_t(A)\}.
\end{equation}
The randomized implementation gives the cleanest one-step expectation identity; deterministic tracking gives approximately the same cumulative allocation.

\subsection{Separation delivered by the selected pair set}
Define the global all-pairs value
\begin{equation}
\label{app:eq:global-separation}
\Gamma_{\rm all}
:=
\max_{\rho\in\Delta(\mathcal A)}
\min_{\theta\neq\vartheta}
\sum_A\rho(A)d_A(\theta,\vartheta).
\end{equation}

\begin{lemma}[Uniform positive all-pairs separation]
\label{lem:uniform-positive-separation}
If $\Theta$ and $\mathcal A$ are finite, $|\Theta|\ge2$, $|\mathcal A|\ge1$, and
Assumption~\ref{ass:falsifiability} holds, then
$\Gamma_{\rm all}>0$.
\end{lemma}

\begin{proof}
Under the uniform allocation $\rho_{\rm unif}(A)=1/|\mathcal A|$, every
distinct pair has strictly positive average discrepancy because at least one
action has positive discrepancy and all terms are nonnegative.  There are
finitely many pairs, so the minimum of these positive averages is positive.
The optimal all-pairs LP value is at least this uniform-design value.
\end{proof}

For the pair set $\mathcal P_t$ actually selected by
Eq.~\eqref{eq:cut-directed-design}, let
\begin{equation}
\label{app:eq:attained-frontier-value}
\Gamma_t^\star
:=
\min_{(\theta,\vartheta)\in\mathcal P_t}
\sum_{A\in\mathcal A}
\rho_{t+1}(A)d_A(\theta,\vartheta).
\end{equation}
Whenever $\mathcal P_t\neq\varnothing$, LP optimality gives
\begin{equation}
\label{app:eq:selected-pair-separation}
\sum_{A\in\mathcal A}
\rho_{t+1}(A)d_A(\theta,\vartheta)
\geq\Gamma_t^\star,
\qquad
(\theta,\vartheta)\in\mathcal P_t.
\end{equation}
If the possibility frontier is empty and $|V_t|>1$, then
$\mathcal P_t=\{(\theta,\vartheta)\in V_t^2:\theta\ne\vartheta\}$, so
$\Gamma_t^\star\geq\Gamma_{\rm all}>0$.  When the frontier is nonempty,
Eq.~\eqref{app:eq:selected-pair-separation} controls only pairs included in the
frontier.  In particular, without a separate exploration mixture, periodic
all-pairs step, or non-starvation condition, falsifiability alone does not
imply positive information for a pair excluded from $\mathcal P_t$.

\begin{proposition}[Frontier-design optimality]
\label{prop:frontier-design-optimality}
Fix a round $t$ and condition on the history immediately before choosing
$A_{t+1}$, so that the finite pair set $\mathcal P_t\neq\varnothing$ is fixed. Define
\begin{equation}
\label{eq:conditional-frontier-value}
\Gamma_t(\rho)
:=
\min_{(\theta,\vartheta)\in \mathcal P_t}
\sum_{A\in\mathcal A}\rho(A)d_A(\theta,\vartheta),
\qquad \rho\in\Delta(\mathcal A).
\end{equation}
If $\rho_{t+1}$ solves Eq.~\eqref{eq:cut-directed-design}, then
\begin{equation}
\label{eq:conditional-frontier-optimality}
\Gamma_t(\widetilde\rho)
\le
\Gamma_t(\rho_{t+1})
=:
\Gamma_t^\star
\qquad
\text{for every }\widetilde\rho\in\Delta(\mathcal A).
\end{equation}
Equivalently, among all randomized one-step experimental designs,
$\rho_{t+1}$ maximizes the worst-pair expected discrepancy over the currently
selected frontier pairs:
\begin{equation}
\min_{(\theta,\vartheta)\in \mathcal P_t}
\mathbb E_{A\sim\widetilde\rho}
\!\left[d_A(\theta,\vartheta)\right]
\le
\min_{(\theta,\vartheta)\in \mathcal P_t}
\mathbb E_{A\sim\rho_{t+1}}
\!\left[d_A(\theta,\vartheta)\right].
\end{equation}
\end{proposition}

\begin{proof}
Conditional on the history, Eq.~\eqref{eq:cut-directed-design} is exactly the
optimization problem
$\max_{\rho\in\Delta(\mathcal A)}\Gamma_t(\rho)$.
Therefore any optimizer $\rho_{t+1}$ satisfies
$\Gamma_t(\widetilde\rho)\le\Gamma_t(\rho_{t+1})$ for every feasible
$\widetilde\rho$. The expectation form follows from
\(
\mathbb E_{A\sim\rho}[d_A(\theta,\vartheta)]
=\sum_A\rho(A)d_A(\theta,\vartheta)
\).
\end{proof}

\subsection{Scope of the maximin optimality statement}
\label{app:scope-optimality}

\paragraph{Relation to classical active testing.}
For a fixed finite hypothesis family, Chernoff's sequential design allocates
experiments through a max--min Kullback--Leibler information criterion and is
asymptotically optimal under its stated regularity conditions
\citep{chernoff1959sequential}. Related active-testing analyses establish
lower bounds in terms of the largest achievable worst-alternative information
rate and construct policies attaining those bounds asymptotically
\citep{naghshvar2013active}. In fixed-confidence best-arm identification,
Track-and-Stop likewise derives an oracle max--min allocation from a tight
sample-complexity lower bound and tracks that allocation, yielding asymptotic
optimality as the confidence level tends to one
\citep{garivier2016optimal}. Nonasymptotic bounds for Chernoff-style active
testing have also been developed by \citet{mukherjee2022chernoff}.

Possibility Frontier Search shares the central experimental-design structure
of these methods: Eq.~\eqref{eq:cut-directed-design} maximizes the minimum
information proxy over a set of currently relevant alternatives. Here those
alternatives are not a fixed hypothesis family, but the data-dependent
possibility frontier $\mathcal P_t$. Proposition~\ref{prop:frontier-design-optimality}
therefore gives an exact finite-round statement: the selected allocation is
maximin optimal for the current frontier criterion. If $d_A$ is the directed
KL divergence, $\Gamma_t^\star$ is the largest worst-pair one-step KL rate
obtainable by a randomized design over $\mathcal A$.

The allocation is maximin-optimal for the current pair set.
Recovery of the complete procedure additionally depends on hypothesis
generation and information acquired across rounds, as captured by
Theorem~\ref{thm:generic-approximate-recovery}.

\subsection{Implementation.}
For a finite pair set $\mathcal P_t=\{(\theta_i,\vartheta_i)\}_{i=1}^n$ and action set $\mathcal A=\{A_j\}_{j=1}^m$, write $D_{ij}:=d_{A_j}(\theta_i,\vartheta_i)$. Prob.~\eqref{eq:cut-directed-design} can be solved by Linear Programming (LP).
\begin{equation}
\begin{aligned}
\label{eq:lp}
&\max_{\rho_1,\cdots,\rho_m,\Gamma} \Gamma\\
\text{s.t.}\quad
     &\sum_{j=1}^m D_{ij} \rho_j \geq \Gamma,\quad \sum_{j=1}^m \rho_j = 1, \quad \rho_j \geq 0.
\end{aligned}
\end{equation}
The solver outputs the numerical allocation
$\rho^\star=(\rho_1^\star,\ldots,\rho_m^\star)$ and its attained worst-pair separation $\Gamma^\star$. Then we draw next action $A_{t+1}$ from the categorical distribution with probabilities $\rho_{t+1}(A_1),\ldots,\rho_{t+1}(A_m)$. (Line~\ref{alg_line:experiment} in Alg.~\ref{alg:possibility-frontier-search}).

\section{Likelihood-Ratio Evidence and Experimental Readiness}
\label{app:likelihood-ratio-rate}
\subsection{Full-support likelihood-ratio mixture}
Assume finite $\Theta$ and $\mathcal A$, with predictive densities
$p_\theta(\cdot\mid A)$ under a common dominating measure. After the
post-expansion history $\mathcal I_{t-1}^+$, action $A_t$ is sampled from
$\rho_t$ and the conditional observation law is
$p_\theta(\cdot\mid A_t)$ under state $\theta$. For each tested state
$\vartheta$, fix weights $w_\vartheta\in\Delta(\Theta)$ such that
\begin{equation}
w_{\min}:=\min_{\theta,\vartheta}w_\vartheta(\theta)>0.
\label{app:eq:full-support-weights}
\end{equation}
Define
\begin{equation}
E_t(\vartheta):=
\sum_{\theta\in\Theta}w_\vartheta(\theta)
\prod_{s=1}^{t}\frac{p_\theta(M_s\mid A_s)}{p_\vartheta(M_s\mid A_s)}.
\label{app:eq:likelihood-ratio-eprocess}
\end{equation}
Ratios are evaluated where the tested density is positive; behavior on
null sets under that tested law can be fixed by an extended-value convention.

\begin{proposition}[Validity under adaptive experiments]
\label{prop:lr-mixture-validity}
Under $\mathbb P_\vartheta$, $E_t(\vartheta)$ is a nonnegative
supermartingale and hence valid statewise evidence. If every alternative
law is absolutely continuous with respect to the tested law for every
action, it is a martingale.
\end{proposition}
\begin{proof}
Conditionally on the history and selected action,
\[
\mathbb E_\vartheta\!\left[
\frac{p_\theta(M_t\mid A_t)}{p_\vartheta(M_t\mid A_t)}
\,\middle|\,\mathcal I_{t-1}^+,A_t\right]
=\int_{\{p_\vartheta>0\}}p_\theta(m\mid A_t)\,dm\le1.
\]
Equality holds under the stated absolute-continuity condition.
Multiplication gives a nonnegative supermartingale for each alternative;
a fixed convex combination has the same property. Conditioning again over
the adaptive action does not change the inequality.
\end{proof}

For a false state $\vartheta$, define
\begin{equation}
L_t(\theta^\star,\vartheta):=
\sum_{s=1}^t\log\frac{p_{\theta^\star}(M_s\mid A_s)}
{p_\vartheta(M_s\mid A_s)}.
\label{app:eq:true-false-llr}
\end{equation}
The true-state term gives
$E_t(\vartheta)\ge w_\vartheta(\theta^\star)e^{L_t(\theta^\star,\vartheta)}$.
Since $\vartheta\in V_t$ implies $E_t(\vartheta)<1/\delta$,
\begin{equation}
\vartheta\in V_t\ \Longrightarrow\
L_t(\theta^\star,\vartheta)<b_\delta,
\qquad b_\delta:=\log\frac1{w_{\min}\delta}.
\label{app:eq:version-space-llr-threshold}
\end{equation}

\subsection{Sub-Gaussian likelihood-ratio rate model}
For the exponential rate, assume finite directed KL divergences and a
uniform conditional sub-Gaussian proxy for log-likelihood-ratio increments
centered by their conditional means given $\mathcal I_{s-1}^+$.
Specifically, writing $X_s(\vartheta)$ for the increment of
$L_s(\theta^\star,\vartheta)$ and
$\mu_s(\vartheta)=\mathbb E_{\theta^\star}[X_s(\vartheta)\mid\mathcal I_{s-1}^+]$,
assume a finite $\sigma^2>0$ such that
\begin{equation}
\mathbb E_{\theta^\star}\!\left[
 e^{a(X_s(\vartheta)-\mu_s(\vartheta))}
 \mid\mathcal I_{s-1}^+\right]
\le e^{a^2\sigma^2/2},\qquad a\in\mathbb R.
\label{app:eq:centered-one-step-llr}
\end{equation}
Take $d_A$ to be directed KL divergence or a nonnegative lower bound.
Then
\[
\mu_s(\vartheta)=\sum_A\rho_s(A)
D_{\mathrm{KL}}(\mathbb P_{\theta^\star}^A\Vert\mathbb P_\vartheta^A)
\ge\sum_A\rho_s(A)d_A(\theta^\star,\vartheta).
\]
Finite action randomization is included in the conditional proxy.
A uniform per-action proxy together with bounded per-action KL means
on the finite spaces gives such a proxy.

For the equal-variance Gaussian observation model
$M=f_\theta(A)+\xi$, $\xi\sim\mathcal N(0,\sigma_E^2)$,
\begin{equation}
\log\frac{p_\theta(M\mid A)}{p_\vartheta(M\mid A)}
=\frac{(f_\theta(A)-f_\vartheta(A))^2}{2\sigma_E^2}
+\frac{f_\theta(A)-f_\vartheta(A)}{\sigma_E^2}\xi.
\label{app:eq:gaussian-llr}
\end{equation}
The per-action centered part is Gaussian and the finite model admits a
uniform proxy, including randomization over actions.

\subsection{Readiness under the actual frontier policy}
Define
\begin{equation}
I_t^\rho(\vartheta):=
\sum_{s=0}^{t-1}\sum_A\rho_{s+1}(A)d_A(\theta^\star,\vartheta).
\label{app:eq:actual-cumulative-information}
\end{equation}
For $\Gamma_s^\star$ in Appendix~\ref{app:frontier-lp}, with value zero
when $\mathcal P_s=\varnothing$, the selected-pair guarantee implies
\begin{equation}
I_t^\rho(\vartheta)\ge
\sum_{s=0}^{t-1}\mathbf1\{(\theta^\star,\vartheta)\in\mathcal P_s\}
\Gamma_s^\star.
\label{app:eq:lp-information-lower-bound}
\end{equation}
It provides no guaranteed contribution from a round in which the relevant
ordered pair is omitted.

Fix $\varepsilon>0$ and a target
$H_\varepsilon\in\mathcal H_\varepsilon^\star$. The non-starvation condition
is that a constant $\Gamma_\varepsilon>0$ satisfies, almost surely on
$\mathcal E_\delta$, for every $t\ge1$ and each
$\vartheta\in(\Theta\setminus H_\varepsilon)\cap V_t$,
\begin{equation}
I_t^\rho(\vartheta)\ge t\Gamma_\varepsilon.
\label{app:eq:frontier-information-rate}
\end{equation}
A sufficient condition is the corresponding linear lower bound in
Eq.~\eqref{app:eq:lp-information-lower-bound} at every required time.
Positive asymptotic pair frequency alone gives an eventual bound and does
not establish this all-time inequality without a separately quantified
initial period. Falsifiability alone does not imply the condition for the
frontier-only policy.

For this fixed target, write $\Gamma=\Gamma_\varepsilon$,
$\sigma=\sigma_\varepsilon>0$ for a valid uniform proxy, and
$m=m_\varepsilon:=|\Theta\setminus H_\varepsilon|$.
On $\mathcal E_\delta\cap\{\vartheta\in V_t\}$, the cumulative conditional
mean is at least $t\Gamma$, while $L_t<b_\delta$. Therefore the centered
sum is less than $b_\delta-t\Gamma$. Iterating
Eq.~\eqref{app:eq:centered-one-step-llr} and using its lower-tail bound gives
\begin{equation}
\mathbb P_{\theta^\star}(\mathcal E_\delta\cap\{\vartheta\in V_t\})
\le\exp\!\left\{-\frac{(t\Gamma-b_\delta)_+^2}{2t\sigma^2}\right\}.
\label{app:eq:false-state-survival}
\end{equation}
This uses event inclusion, not a concentration assertion conditional on
$\mathcal E_\delta$. A union bound over the $m$ states outside the target
now gives
\begin{equation}
\mathbb P_{\theta^\star}
(\mathcal E_\delta\cap\{V_t\nsubseteq H_\varepsilon\})
\le R_t(\varepsilon),
\label{app:eq:readiness-rate}
\end{equation}
where $R_0(\varepsilon)=1$ and
\begin{equation}
R_t(\varepsilon)=1\wedge m
\exp\!\left\{-\frac{(t\Gamma-b_\delta)_+^2}{2t\sigma^2}\right\},
\quad t\ge1.
\label{app:eq:readiness-envelope}
\end{equation}
The quantity $m$ counts states outside a chosen target, not the number of
near-optimal hypotheses.

\section{Generator-Relative Recovery}
\label{app:recovery-details}
\subsection{Reachable class and target}
Analyze the generator and policy with the stopping check and finite
observation cap disabled; its per-round proposal budgets remain specified.
The stopped implementation agrees with this continuation until termination.
For a discrete semantic representation, define
\begin{equation}
\mathcal H_\infty:=\mathcal H_0\cup
\{H:\mathbb P_{\theta^\star}^{\mathsf{Gen},\rho}
(\exists t\ge0,\,1\le j\le B_t:
\text{attempt }j\text{ is reached and }\widetilde H_{t,j}=H)>0\}.
\label{app:eq:reachable-class}
\end{equation}
Only valid nonempty proper claims enter this class. In a finite or countable
representation, all realized claims belong to this class almost surely.
For a more general representation, specify a measurable comparison class
containing all proposals almost surely; positive point mass is not an
appropriate reachability definition for an atomless generator. All
optimizations and events must be measurable. The finite-state recovery
analysis uses Eq.~\eqref{app:eq:reachable-class}.

Fix $\varepsilon$ and recall
$\mathcal H_\varepsilon^\star$ and $W_{t,\varepsilon}$ from
Eqs.~\eqref{eq:epsilon-optimal-generator-set},~\eqref{app:eq:readiness-event}.
On $\mathcal E_\delta$, readiness persists because the version space
shrinks while retaining $\theta^\star$. It occurs whenever
$V_t\subseteq H_\varepsilon$, so
\begin{equation}
\mathbb P_{\theta^\star}
(\mathcal E_\delta\cap W_{t,\varepsilon}^c)
\le R_t(\varepsilon).
\label{app:eq:readiness-error}
\end{equation}

\subsection{Proof of Theorem~\ref{thm:generic-approximate-recovery}}
Fix $0\le s\le t$ and write $q=q_\varepsilon$. Define the post-expansion
and pre-expansion failure events
\[
N_r:=\{\mathsf C_r\cap\mathcal H_\varepsilon^\star=\varnothing\},
\qquad
N_r^-:=\{\mathsf C_r^-\cap\mathcal H_\varepsilon^\star=\varnothing\},
\]
and the finite-history coverage event
\[
\mathcal E_{\delta,r}:=
\{\theta^\star\in V_j\text{ for }0\le j\le r\}\in\mathcal I_r.
\]
For $r\ge s$, let
\[
B_r:=\mathcal E_{\delta,r}\cap W_{s,\varepsilon}
\cap\bigcap_{j=s}^{r}N_j.
\]
For the base step, the pre-expansion event
$D_s:=\mathcal E_{\delta,s}\cap W_{s,\varepsilon}\cap N_s^-$
satisfies
$B_s\subseteq D_s\cap\{Z_{s,\varepsilon}=0\}$.
Indeed, a previously certified true target cannot disappear during
expansion. Assumption~\ref{ass:generator-coverage} therefore gives
$\mathbb P(B_s)\le1-q$.

For $r>s$, define
\[
D_r:=B_{r-1}\cap\mathcal E_{\delta,r}\cap N_r^-.
\]
It is $\mathcal I_r$-measurable and implies $W_{r,\varepsilon}$: a target
certifiable at $s$ is still certifiable after shrinking the nonempty
version space. Since
$B_r\subseteq D_r\cap\{Z_{r,\varepsilon}=0\}$,
\[
\mathbb P(B_r)
\le(1-q)\mathbb P(D_r)
\le(1-q)\mathbb P(B_{r-1}).
\]
Iteration yields $\mathbb P(B_t)\le(1-q)^{t-s+1}$.
On $\mathcal E_\delta$, a certified target is retained forever; hence
$\mathcal E_\delta\cap W_{s,\varepsilon}\cap N_t\subseteq B_t$ and
\begin{equation}
\mathbb P_{\theta^\star}
(\mathcal E_\delta\cap W_{s,\varepsilon}\cap N_t)
\le(1-q_\varepsilon)^{t-s+1}.
\label{app:eq:geometric-generation-error}
\end{equation}
This is a joint-event bound. The proposal kernel is never conditioned on
the future-dependent event $\mathcal E_\delta$.

Combining this bound with Eq.~\eqref{app:eq:readiness-error} gives
\begin{equation}
\mathbb P_{\theta^\star}(\mathcal E_\delta\cap N_t)
\le\inf_{0\le s\le t}
\{R_s(\varepsilon)+(1-q_\varepsilon)^{t-s+1}\}
=\mathcal B_t(\varepsilon).
\label{app:eq:joint-recovery-error}
\end{equation}
Adding $\mathbb P(\mathcal E_\delta^c)\le\delta$ proves the actual-target
statement in the theorem. The event $\{U_t<U^\star-\varepsilon\}$ is
contained in $N_t$, proving the scalar statement as well.

\subsection{Asymptotic consequence}
For $s=\lfloor t/2\rfloor$ and sufficiently large $t$,
\[
R_s(\varepsilon)
\le m_\varepsilon e^{-s\Gamma_\varepsilon^2/(8\sigma_\varepsilon^2)},
\qquad
(1-q_\varepsilon)^{t-s+1}
\le e^{-q_\varepsilon(t-s+1)}.
\]
Thus $\mathcal B_t(\varepsilon)\to0$. On $\mathcal E_\delta$, a certified
target never disappears, so $\mathcal E_\delta\cap N_t$ is a decreasing
sequence of events. By Eq.~\eqref{app:eq:joint-recovery-error}, its
intersection has probability zero. Hence a target is eventually obtained
and retained almost surely on $\mathcal E_\delta$, and
\begin{equation}
\mathbb P_{\theta^\star}\!\left(
\mathcal E_\delta\cap
\{\exists t_0<\infty:\forall t\ge t_0,
\mathsf C_t\cap\mathcal H_\varepsilon^\star\ne\varnothing\}\right)
\ge1-\delta.
\label{app:eq:eventual-recovery}
\end{equation}
The conclusion is pointwise in $\varepsilon$. It is not inferred from
an unconditional limiting failure probability of $\delta$ alone, and it
does not imply continued refinement after the actual procedure has stopped.

\subsection{Observable stopping guarantee}
\label{app:observable-abductive-stopping}
The stopping certificate uses the fixed contour and an observable upper
bound, rather than unknown generation or information-rate constants.

\begin{theorem}[Certified primary optimality at stopping]
\label{thm:observable-abductive-stopping}
Assume simultaneous version-space coverage, a fixed background contour,
a finite current archive contained almost surely in the comparison class,
and measurable set evaluations, scores, and the final tie-break.
The comparison class is defined with termination disabled as above.
Then, on $\mathcal E_\delta$, simultaneously for every round with
$\mathsf C_t\ne\varnothing$,
\[
\theta^\star\in\widehat H_t,
\qquad
0\le U^\star-U_t\le\overline U_t-U_t.
\]
The guarded stopping time in Eq.~\eqref{eq:abductive-stopping-time}
satisfies
\begin{equation}
\begin{split}
\mathbb P_{\theta^\star}\bigl(
&\tau_\varepsilon<\infty,\ 
[\theta^\star\notin\widehat H_{\tau_\varepsilon}\ \text{or}\\
&u(\widehat H_{\tau_\varepsilon})<U^\star-\varepsilon]
\bigr)\le\delta.
\end{split}
\label{app:eq:observable-stopping-guarantee}
\end{equation}
No generator-coverage or non-starvation constant is required for this
stopping guarantee.
\end{theorem}
\begin{proof}
Work on $\mathcal E_\delta$. Every certified claim contains
$\theta^\star$ and belongs to the comparison class; hence $U_t\le U^\star$
and the selected claim is true. Every true comparison claim satisfies
$u(H)\le\pi^K(\theta^\star)$. Since $\theta^\star\in V_t$,
\[
U_t\le U^\star\le\pi^K(\theta^\star)
\le\sup_{\theta\in V_t}\pi^K(\theta)=\overline U_t.
\]
All stopping checks are measurable with respect to $\mathcal I_t^+$,
so the first-passage rule is a stopping time for the post-expansion
filtration. At termination on $\mathcal E_\delta$,
\[
U^\star-u(\widehat H_{\tau_\varepsilon})
=U^\star-U_{\tau_\varepsilon}
\le\overline U_{\tau_\varepsilon}-U_{\tau_\varepsilon}
\le\varepsilon.
\]
Thus the failure event is contained in $\mathcal E_\delta^c$.
The additional guard $\mathsf O_t^+=\varnothing$ only restricts when the
report occurs and does not change this inclusion.
\end{proof}

The bound is joint in termination and failure; it is not a conditional
error bound given termination. It certifies primary utility, not globally
minimal cardinality or optimal experimental cost. Termination need not
occur: on $\mathcal E_\delta$,
$\overline U_t-U_t\ge\pi^K(\theta^\star)-U^\star$, which can stay positive
for a restricted language. A previously ungenerated refinement may also
exist after the current active set is empty.

Numerical implementation must use exact extrema or conservative bounds:
an upper bound on the version-space supremum and a lower bound on the
selected certified claim's utility. If their difference is at most
$\varepsilon$, the same event-inclusion proof applies. Sampling a subset
of states does not in general give bounds in these required directions.

\subsection{Stopping and sample complexity with known constants}
If the recovery constants are known valid values, define
\begin{equation}
T_{\varepsilon,\zeta}:=
\inf\{t:\mathcal B_t(\varepsilon)\le\zeta\},
\label{app:eq:oracle-recovery-horizon}
\end{equation}
and the first reporting time after that horizon,
\begin{equation}
\tau_{\varepsilon,\zeta}:=
\inf\{t\ge T_{\varepsilon,\zeta}:\mathsf C_t\ne\varnothing\}.
\label{app:eq:oracle-reporting-time}
\end{equation}
On $\mathcal E_\delta$, a certified target at the deterministic horizon
persists. Using Eq.~\eqref{app:eq:joint-recovery-error} therefore gives
\begin{equation}
\mathbb P_{\theta^\star}\!\left(
\tau_{\varepsilon,\zeta}<\infty,\;
[\theta^\star\notin\widehat H_{\tau_{\varepsilon,\zeta}}
\ \text{or}\ 
 u(\widehat H_{\tau_{\varepsilon,\zeta}})<U^\star-\varepsilon]
\right)\le\delta+\zeta.
\label{app:eq:oracle-horizon-guarantee}
\end{equation}
This ex ante horizon is distinct from the observable stopping certificate;
it need not be computable for a free-form generator.

\begin{corollary}[Explicit experimental sample complexity]
\label{cor:explicit-sample-complexity}
Under Theorem~\ref{thm:generic-approximate-recovery}, fix $\varepsilon>0$
and $\zeta\in(0,1)$. In this corollary, abbreviate
$q=q_\varepsilon$, $\Gamma=\Gamma_\varepsilon$,
$\sigma=\sigma_\varepsilon$, and $m=m_\varepsilon\ge1$.
Define
\begin{align}
n_E&:=\left\lceil\max\left\{1,\frac{2b_\delta}{\Gamma},
\frac{8\sigma^2}{\Gamma^2}\log\frac{2m}{\zeta}\right\}\right\rceil,
\label{eq:explicit-readiness-horizon}\\
n_G&:=\begin{cases}
\left\lceil\dfrac{\log(2/\zeta)}{-\log(1-q)}\right\rceil,&q<1,\\
1,&q=1,
\end{cases}
\label{eq:explicit-generation-horizon}\\
\overline T_{\varepsilon,\zeta}&:=n_E+n_G-1.
\label{eq:explicit-total-horizon}
\end{align}
Then
\begin{equation}
\mathbb P_{\theta^\star}
(\mathsf C_{\overline T_{\varepsilon,\zeta}}
\cap\mathcal H_\varepsilon^\star=\varnothing)
\le\delta+\zeta,
\label{eq:explicit-sample-complexity-guarantee}
\end{equation}
and $T_{\varepsilon,\zeta}\le\overline T_{\varepsilon,\zeta}$. Moreover,
\begin{equation}
\overline T_{\varepsilon,\zeta}
=O\!\left(
1+\frac{b_\delta}{\Gamma}
+\frac{\sigma^2}{\Gamma^2}\log\frac{2m}{\zeta}
+\frac1q\log\frac2\zeta\right).
\label{eq:explicit-sample-complexity-rate}
\end{equation}
\end{corollary}
\begin{proof}
By construction, $n_E\Gamma\ge2b_\delta$, whence
\[
\frac{(n_E\Gamma-b_\delta)^2}{2n_E\sigma^2}
\ge\frac{n_E\Gamma^2}{8\sigma^2}
\ge\log\frac{2m}{\zeta}.
\]
Therefore $R_{n_E}(\varepsilon)\le\zeta/2$. Also
$(1-q)^{n_G}\le\zeta/2$, including $q=1$ because $n_G=1$.
The split at $n_E$ yields
\[
\mathcal B_{\overline T_{\varepsilon,\zeta}}(\varepsilon)
\le R_{n_E}(\varepsilon)+(1-q)^{n_G}\le\zeta.
\]
Apply the recovery theorem and the definition of the deterministic horizon.
The order bound follows by bounding a maximum by a sum, controlling the
ceilings by an additive constant, and using $-\log(1-q)\ge q$ for $q<1$.
\end{proof}

\section{Implementation Notes}
\label{app:implementation}
\subsection{Natural-language hypotheses and reusable tests}
A generated sentence must compile into a measurable claim within a
registered experimental interface. Statistical validity comes from the
statewise evidence, not from the calibration of the language model or
background assessor. For example, a post-hoc claim ``conductivity decreases
with temperature'' can be represented by $H=\{\beta<0\}$. The retained
statewise evidence then evaluates the complement through
$\Pi_t(H^c)=\sup_{\beta\ge0}\pi_t(\beta)$.

\paragraph{Composite interpretation of qualitative states.}
Let $\mathfrak P$ denote the distributions supported by the registered
interface. For registered population quantities, let
$g:\mathfrak P\to\mathcal Z\subseteq\{-,0,+\}^J$ record their signs.
The qualitative state $z$ represents the composite cell
$\mathfrak P_z:=\{P\in\mathfrak P:g(P)=z\}$, not just one numerical
representative. A compiled claim asserts $g(P)\in H$ for some
$H\subseteq\mathcal Z$. The abstract state space in this implementation
is the qualitative label space; the numerical effect sizes remain
continuous within its cells.

For experiment $s$, let $\mathcal J_s$ include the post-expansion history,
selected action, and frozen test specification, before the fresh data are
opened. Let $\mathfrak N_s$ be its null class. A sufficient condition is
conditional super-uniformity: for every supported $P$,
\[
\mathbb P_P(p_s\le x\mid\mathcal J_s)\le x,
\qquad 0\le x\le1,
\]
almost surely on $\{P\in\mathfrak N_s\}$. For fixed $0<\kappa<1$, use
\[
e_s(z)=
\begin{cases}
\kappa p_s^{\kappa-1},&\mathfrak P_z\subseteq\mathfrak N_s,\\
1,&\text{otherwise},
\end{cases}
\qquad E_t(z)=\prod_{s=1}^te_s(z).
\]
The entire composite cell must imply the frozen null. Testing only a
representative's numerical value is insufficient.

\begin{proposition}[Validity of the compiled interface]
\label{prop:compiled-interface}
Under the stated conditional null validity, $E_t(z)$ is a nonnegative
supermartingale under every $P\in\mathfrak P_z$. The resulting persistent
version space contains $g(P)$ simultaneously over time with probability
at least $1-\delta$. Adaptively generated claims containing that nonempty
version space inherit simultaneous certification validity.
\end{proposition}
\begin{proof}
Fix $P\in\mathfrak P_z$. When the cell implies the selected null,
conditional super-uniformity and the decreasing p-to-e calibrator
\citep{vovk2021values} give
\[
\mathbb E_P[\kappa p_s^{\kappa-1}\mid\mathcal J_s]
\le\int_0^1\kappa x^{\kappa-1}\,dx=1.
\]
Otherwise, the multiplier is one. Multiplication and iterated conditional
expectation prove the supermartingale property. The statewise inequality
and version-space containment give the certification statement.
\end{proof}
Fresh folds address reuse and adaptive specification; they do not prove
conditional calibration of the selected test. If only an asymptotic test
is implemented, the empirical procedure must be labeled approximate.
Validity of the composite interface also does not establish the finite-state
likelihood-ratio information rate: opposite-sign distribution classes may
contain arbitrarily small effects.

\paragraph{Fixed background contour.}
Guaranteed possibility requires one shared background contour over the
registered qualitative states, or an equivalent representation supporting
the same minima. Independent assessments of $H$ and $H^c$ do not determine
$\inf_{z\in H}\pi^K(z)$. The construction, normalization, representation,
and exact or bounded minimization procedure must be documented separately
from evidential validity. The qualitative objective concerns registered
sign configurations; it is not automatically a worst-case guarantee over
all quantitative effect sizes inside a sign cell.

\subsection{Nonunique explanations}
For finite represented states, the implementation returns one claim by
Eq.~\eqref{eq:specificity-tiebreak}, followed by a fixed final ordering.
Semantic paraphrases share one representation and cardinality.
The recovery theorem concerns $U_t$; it does not claim global recovery of
the smallest utility maximizer. The HMS of the single returned claim must
not be replaced by an average over a different argmax set.

\subsection{Computational cost}
Let $S=|V_t|$, $M=|\mathcal H_{\le t}|$, and $L=|\mathcal A|$.
Classifying the retained archive costs $O(SM)$ membership queries.
The frontier has at most $O(S^2M)$ raw ordered pairs before deduplication,
and at most $S(S-1)$ distinct pairs. The LP has $L+1$ variables and one
inequality per distinct constrained pair. These costs assume that claim
scores and memberships are available; constructing the background contour
and computing global extrema can be additional costs.

Conditional i.i.d. rejection sampling needs $1/p_t^{\mathrm{cert}}$
attempts on average without a cap, whereas Algorithm~\ref{alg:expand}
uses at most $B_t$ attempts. A raw-proposal log, including rejected and
duplicate proposals, is required for that accounting. An experiment is
not added to the evidence record twice merely because a generated claim
is re-evaluated or paraphrased.

\section{Additional Experimental Details and Results}
\label{app:experiments}

This section provides additional experimental setups, implementation details, and supplementary results.

\subsection{Synthetic experiment} \label{app:synthetic}

\paragraph{Configuration.}
We define the background environment as
\[
W_K(A)
=
0.55\sin(2\pi A)+0.18\cos(4\pi A).
\]
The background dataset $\mathcal D_K$ consists of observations
at $100$ uniformly sampled locations $A\in[0,1]$, with
Gaussian observation noise $\sigma_K=0.08$.
We model the background knowledge using a zero-mean Gaussian
process with RBF kernel
\[
k(A,A')
=
\exp\left(
-\frac{(A-A')^2}{2\ell^2}
\right),
\]
where $\ell=0.04$. The background dataset and GP
hyperparameters are fixed across simulations.

The parameter space is
\[
\Theta=[0.19,1.81]\times[0.19,1.00],
\]
discretized with step size $0.03$, giving $1{,}540$ states.
We construct the reachable hypothesis class
$\mathcal H_\infty$ from $44$ overlapping rectangular
hypotheses
\[
H_{rs}
=
(I_r^a\times I_s^w)\cap\Theta.
\]
The amplitude intervals are
\[
I_r^a=[c_r-0.18,c_r+0.18],
\qquad
c_r\in
\{0.37,0.49,0.61,0.73,0.85,0.94,
1.09,1.21,1.33,1.45,1.57\},
\]
and the width intervals are
\[
I_s^w\in
\{
[0.19,0.40],
[0.31,0.52],
[0.43,0.64],
[0.64,0.85]
\}.
\]
The ground-truth parameter is
$\theta^\star=(1.00,0.28)$, and we define the set of
true reachable hypotheses as
\[
\mathcal H^\star
=
\{H\in\mathcal H_\infty:\theta^\star\in H\}.
\]
In this construction, $|\mathcal H^\star|=3$.

For the incomplete-initialization setting,
$\mathcal H_0$ consists of the four hypotheses nearest to a
starting point $(a_0,w_0)$ under normalized Euclidean
distance between hypothesis centers. We use starting points
from the Cartesian product
\[
a_0\in\{0.35,0.45,0.55,0.65\},
\qquad
w_0\in\{0.64,0.70,0.76,0.82\},
\]
yielding $16$ initial candidate sets, each satisfying
$\mathcal H_0\cap\mathcal H^\star=\varnothing$.

\paragraph{Evidence and Knowledge Models.}
At round $t$, selecting an action $A_t\in\mathcal A$ produces
\[
M_t
=
W_{\theta^\star}(A_t)+\varepsilon_t,
\qquad
\varepsilon_t\sim\mathcal N(0,\sigma_E^2),
\]
with $\sigma_E=0.02$ and an experimental budget of
$T=16$ rounds; for ablation, we increase this to $T=30$.
The action space $\mathcal A$ consists of $73$ evenly spaced
locations on $[0,1]$.
We select experiments using the possibility-frontier maximin
design in Eq.~\eqref{eq:cut-directed-design}, with the
Gaussian KL discrepancy
\[
d_A(\theta,\vartheta)
=
\frac{
\bigl(W_\theta(A)-W_\vartheta(A)\bigr)^2
}{
2\sigma_E^2
}.
\]
If the possibility frontier is empty, the design maximizes
the minimum separation over all distinct surviving state
pairs in $V_t$.
If $|V_t|\leq 1$, we instead use the uniform distribution
over $\mathcal A$.

Under the Gaussian observation model, the likelihood after
$t$ experiments is
\[
L_t(\theta)
=
\prod_{i=1}^t
\mathcal N\!\left(
M_i;
W_\theta(A_i),
\sigma_E^2
\right).
\]
Following the method-of-mixtures construction for sequential
likelihood ratios~\citep{howard2021time}, for each
$\theta=(a,w)$ we use fixed Gaussian mixture weights over
alternative states $\vartheta=(a',w')$,
\[
q_\theta(\vartheta)
\propto
\exp\left[
-\frac12
\left(
\frac{(a'-a)^2}{\tau_a^2}
+
\frac{(w'-w)^2}{\tau_w^2}
\right)
\right],
\]
where $\tau_a=0.10$ and $\tau_w=0.06$.
Since $\Theta$ is finite in the experiment, $q_\theta$ is
normalized over the full discretized state space and the
resulting $e$-process is
\[
E_t(\theta)
=
\sum_{\vartheta\in\Theta}
q_\theta(\vartheta)
\frac{L_t(\vartheta)}{L_t(\theta)}.
\]
The persistent evidential possibility contour is obtained using Eq. \ref{eq:epos_theta}.
We use $\delta=0.1$ and classify hypotheses using the
version space and certification criteria in Eq. \ref{eq:three-sets}.

For background knowledge, each candidate environment is
evaluated at $31$ evenly spaced locations, yielding the
vector $Z_\theta$.
The conditioned GP induces
\[
Z\mid\mathcal D_K
\sim
\mathcal N(\mu_K,\Sigma_K).
\]
Motivated by Gaussian possibility constructions and
likelihood normalization
~\citep{ristic2020target,denoeux2014likelihood}, we define
\[
s_K(\theta)
=
\exp\left[
-\frac12
(Z_\theta-\mu_K)^\top
\widetilde\Sigma_K^{-1}
(Z_\theta-\mu_K)
\right],
\]
where
\[
\widetilde\Sigma_K
=
\kappa(\Sigma_K+10^{-6}I),
\qquad
\kappa=300.
\]
The fixed background possibility contour is obtained by
supremum normalization,
\[
\pi^K(\theta)
=
\frac{s_K(\theta)}
{\max_{\vartheta\in\Theta}s_K(\vartheta)}.
\]
Each nonempty hypothesis is scored by the background
compatibility utility, which is a minimum on the finite discretized state space.
Among certified hypotheses, we select a maximizer of
$u(H)$.

\paragraph{Hypothesis Expansion.}
Starting from $\mathcal H_0$, we expand the hypothesis set by proposing up to $B_t=4$ previously unexplored hypotheses per round, ordered by decreasing posterior mass. This prioritizes hypotheses that are most plausible given both the accumulated experimental evidence and the background knowledge model. Expansion stops early once a proposed hypothesis is certified; otherwise, proposals continue until the round budget is exhausted.

\paragraph{Baseline Implementation.}
MLE selects the hypothesis with the highest profile
likelihood, while MAP selects the hypothesis with the
largest posterior mass. Both methods remain restricted to
the initial candidate set $\mathcal H_0$.

For MAP and AutoDiscovery, we construct a discrete Bayesian
prior from the same background knowledge score,
\[
p_K(\theta)
=
\frac{s_K(\theta)}
{\sum_{\vartheta\in\Theta}s_K(\vartheta)},
\]
and update it after observing $\mathcal D_t$ according to
\[
p_t(\theta\mid\mathcal D_t)
=
\frac{p_K(\theta)L_t(\theta)}
{\sum_{\vartheta\in\Theta}
p_K(\vartheta)L_t(\vartheta)}.
\]
The corresponding prior and posterior mass of a hypothesis
$H$ are
\[
q_K(H)
=
\sum_{\theta\in H}p_K(\theta),
\qquad
q_t(H)
=
\sum_{\theta\in H}
p_t(\theta\mid\mathcal D_t).
\]

AutoDiscovery scores hypotheses using Bayesian surprise,
\[
S_t(H)
=
D_{\mathrm{KL}}\!\left(
\operatorname{Bern}(q_t(H))
\,\|\,
\operatorname{Bern}(q_K(H))
\right).
\]
To expand the hypothesis space, AutoDiscovery maintains a
persistent Monte Carlo tree search (MCTS) across rounds.
Tree traversal uses an upper-confidence-tree criterion based
on Bayesian surprise, with progressive widening controlling
the number of explored children. New children are sampled
from previously unexplored hypotheses, with the proposal
distribution favoring hypotheses near the current branch in
normalized parameter space. At each round, MCTS performs
four such expansions, after which AutoDiscovery returns the
discovered hypothesis with the largest Bayesian surprise.

MLE, MAP, and AutoDiscovery are evaluated on the same
uniformly sampled experimental observations, while our
method selects experiments adaptively using the
possibility-frontier design.

\paragraph{Evaluation.}
Hypothesis identification is evaluated by precision with
respect to the true reachable hypotheses
$\mathcal H^\star$.
For all methods, precision is determined by whether the
selected hypothesis $\hat H_t$ belongs to $\mathcal H^\star$.
Rounds in which a method returns no hypothesis are excluded
from the precision average.

Abductive quality is evaluated using the utility
$u(H)$ defined in Eq.~\eqref{eq:guaranteed-background}. For our method,
we report $U_t=\sup_{H\in\mathcal C_t}u(H)$
with $U_t=0$ when $\mathcal C_t=\emptyset$.
For MLE, MAP, and AutoDiscovery, the utility of the selected
hypothesis is reported only when that hypothesis satisfies
the same evidence-certification criterion; otherwise it is
set to zero.

Results are averaged over eight random seeds and 16 initial
candidate sets. Uncertainty is computed from seed-level
averages.

\paragraph{Additional Results}

We additionally report the certification rate
$\Pr(\mathcal C_t\neq\varnothing)$ and false-certification
rate
$\Pr(\mathcal C_t\not\subseteq\mathcal H^\star)$
in Fig.~\ref{fig:exp2_selection} for both
incomplete and complete initialization settings.
Our method abstains when no hypothesis is certified and
achieves zero empirical false certifications across the
evaluated runs, while the certification rate increases as
evidence accumulates.

\begin{figure}[h]
    \centering
    \begin{subfigure}[t]{0.48\linewidth}
        \centering
        \includegraphics[width=\linewidth]{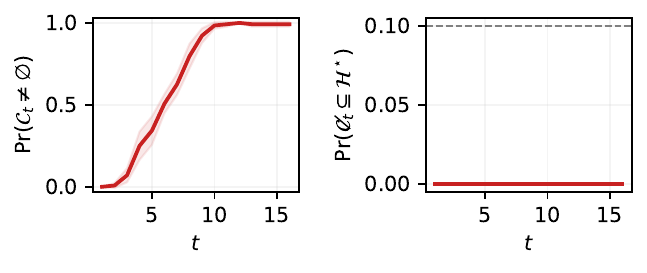}
        \caption{Initial initial hypothesis set.}
        \label{fig:left}
    \end{subfigure}
    \hfill
    \begin{subfigure}[t]{0.48\linewidth}
        \centering
        \includegraphics[width=\linewidth]{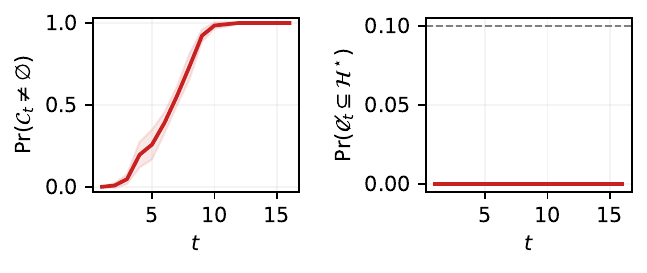}
        \caption{Complete initial hypothesis set.}
        \label{fig:right}
    \end{subfigure}
    \caption{As evidence accumulates, hypotheses are certified only when sufficiently supported, with the false-certification rate remaining below the target level $\delta=0.1$. Shaded regions show 95\%
confidence intervals}
    \label{fig:exp2_selection}
\end{figure}


\subsection{Real-World Dataset Experiment}
\label{app:natural_language_discovery}


\paragraph{Initial hypotheses.}
Each run begins from an initial candidate set
$\mathcal H_0=\{H_0,H_0^c\}$ consisting of one
predefined natural-language hypothesis and its logical
complement. Neither hypothesis reveals the benchmark gold
hypothesis. For each dataset, we evaluate five different
initial hypothesis pairs that vary the relationship,
population, or explanatory variable under consideration.
All methods use the same initial pair for a given run.

\paragraph{Experimental state space and interface.}
The state $\theta\in\Theta$ specifies a
possible model of the world through the observation laws
$\{P_\theta(\cdot\mid A):A\in\mathcal A\}$, while an action
$A\in\mathcal A$ specifies a measurement that can distinguish
between such states.

For DiscoveryBench, we represent the experimentally
accessible part of $\theta$ by the population quantities
measurable through the interface. Let
$j=1,\ldots,J$ index these quantities, or \emph{estimands},
and write
\[
\theta=(\theta_1,\ldots,\theta_J),
\]
where $\theta_j$ denotes the value of estimand $j$ under
state $\theta$. For example, an estimand may be a difference
in population means or a population regression coefficient.
The corresponding unknown true state is
\[
\theta^\star
=
(\theta_1^\star,\ldots,\theta_J^\star).
\]

Each action $A$ measures one estimand; let $j_A$ denote its
index. Under state $\theta$, the population quantity targeted
by $A$ is therefore $\theta_{j_A}$. Different actions may
target the same estimand, for example through fresh-fold
replications or different supported statistical procedures.
States that agree on all estimands accessible through
$\mathcal A$ are observationally indistinguishable under
this interface and need not be represented separately.

For possibility-frontier computation, we use a finite
working representation in which each estimand takes one of
three qualitative states: negative, zero, or positive. We
assign these states the numerical representatives
\[
\theta_j
\in
\{-s_j^{\mathrm{ref}},0,s_j^{\mathrm{ref}}\},
\]
where $s_j^{\mathrm{ref}}$ is a prespecified scale for
estimand $j$, chosen as a typical standard error for that
estimand family. The nonzero values thus provide
standardized representatives of negative and positive
departures from zero.

The resulting working state space is
\[
\Theta_{\mathrm{grid}}
=
\prod_{j=1}^{J}
\{-s_j^{\mathrm{ref}},0,s_j^{\mathrm{ref}}\}.
\]
This discretization is used only for frontier computation
and does not assume that the underlying population
quantities are discrete.


Each generated hypothesis is mapped to a sign condition on its corresponding estimand. For example, consider the hypothesis that completing a bachelor's degree is positively associated with later socioeconomic status. Let \(\theta_j\) denote the corresponding population regression coefficient. The hypothesis is represented as
\[
H=\{\theta:\theta_j>0\},
\qquad
H^c=\{\theta:\theta_j\leq0\}.
\]
An associated action estimates this coefficient on a fresh data fold using the prespecified regression procedure and tests the corresponding directional null. In the working grid, \(H\) contains states with \(\theta_j=s_j^{\mathrm{ref}}\), whereas \(H^c\) contains those with \(\theta_j\in\{-s_j^{\mathrm{ref}},0\}\). 
The remaining coordinates encode the signs of the other estimands registered in the fixed experimental interface \(\mathcal A\).

\paragraph{Statistical procedures.}
We adapt POPPER~\citep{huang2025popper} to instantiate candidate actions for generated hypotheses and their complements. POPPER constructs each action as a falsification test whose null hypothesis is implied by the hypothesis being tested, so that rejection provides evidence against that hypothesis.

At each round, POPPER proposes a finite set of candidate actions from the predefined action space $\mathcal A$ for the hypotheses generated so far and their complements. Each action specifies its estimand, population or subgroup, variables and controls, null and alternative hypotheses, and statistical test before its assigned data fold is observed. We retain only proposals whose full specification is supported by the prespecified experimental interface; unsupported or duplicate proposals are discarded before any new data fold is observed. For example, in NLS SES, supported actions test directional associations between bachelor's degree completion and socioeconomic status, family size, academic ability, or class rank. We use one-sided OLS coefficient tests with HC3 robust standard errors, including predefined subgroup analyses.

\paragraph{Data splitting and experimental budget.}
For each dataset, we randomly partition the observations
into 16 disjoint folds using the run's random seed.
At round $t$, we select up to two distinct actions from
$\mathcal A_t$. Their complete statistical specifications
are fixed before the corresponding fold is observed, and
the selected actions are then executed on that same
previously unseen fold. Thus, each run uses at most 16
data folds and permits at most 32 test executions.

We retain the specifications and results of all executed
actions in a shared evidence record within each run.
Previously executed actions cannot be introduced again
merely by changing their wording. A registered action may,
however, be replicated on a later fresh fold without
modifying its estimand, population, null and alternative
hypotheses, or statistical procedure. We permit at most
two such replications of each registered action.

\paragraph{Evidence construction.}
For each selected action $A$, POPPER executes its
prespecified statistical test on the round's previously
unseen data fold and returns the $p$-value for its fixed
null hypothesis. Following POPPER~\citep{huang2025popper},
we convert each $p$-value to an $e$-value using
\[
e(p)
=
\kappa p^{\kappa-1},
\qquad
\kappa=0.5.
\]
The resulting evidence is stored together with the action's
estimand $j_A$, null hypothesis, and statistical
specification.

Recall that action $A$ measures estimand $j_A$, whose value
under state $\theta$ is $\theta_{j_A}$. We therefore map the
null hypothesis of each action to the states in
$\Theta_{\mathrm{grid}}$ for which the corresponding
condition on $\theta_{j_A}$ holds. Let $A_{t,k}$ denote the
$k$th action executed on fold $t$, let $e_{t,k}$ denote its
$e$-value, and let $\Theta_{0,t,k}$ denote the set of states
satisfying its prespecified null hypothesis. Its statewise
contribution is
\[
e_{t,k}(\theta)
=
e_{t,k}\,\mathbf{1}\{\theta\in\Theta_{0,t,k}\}
+
\mathbf{1}\{\theta\notin\Theta_{0,t,k}\}.
\]

At each round, up to two distinct actions may be executed on
the same data fold. These actions may target different
estimands, but because their test statistics are computed
from the same observations they need not be independent.
We therefore average their statewise contributions within
the fold,
\[
\bar e_t(\theta)
=
\frac{1}{K_t}
\sum_{k=1}^{K_t} e_{t,k}(\theta),
\qquad
K_t\leq 2.
\]
All selected actions and their statistical specifications
are fixed before fold $t$ is observed. Evidence is then
accumulated across successive disjoint folds as
\[
E_t(\theta)
=
\prod_{s=1}^t \bar e_s(\theta).
\]


\paragraph{Possibility frontier.}
Using the current version space $V_t$, we construct the
frontier pair set $\mathcal P_t$ as in
Sec.~\ref{sec:pos_front}. The pairs in $\mathcal P_t$
represent distinctions between surviving states that
separate open hypotheses in $\mathcal O_t^+$ and could
therefore improve on the current certified utility $U_t$.
If $\mathcal O_t^+=\varnothing$, we instead consider all
distinct pairs of states in $V_t$.

\paragraph{Experiment selection.}
At each round, POPPER instantiates candidate actions from
the supported experimental interface $\mathcal A$ for hypotheses generated so far and their complements,
including eligible replications of previously registered
actions. We apply the maximin design of
Eq.~\eqref{eq:cut-directed-design} to the resulting
validated actions.

Recall that each action
$A$ measures estimand $j_A$.
For experimental design only, we use the working observation
model
\[
\widehat{\theta}_A\mid\theta
\approx
\mathcal N\!\left(
\theta_{j_A},\sigma_A^2
\right),
\]
where $\sigma_A$ is the anticipated standard error of
action $A$. If estimand $j_A$ has been measured on an
earlier fold, we use its previously observed standard
error; otherwise, we use a prespecified reference standard error
for the corresponding estimand family, such as $0.08$ for
binary prevalence estimands. Thus, $\sigma_A$ is always
determined before the current fold is observed.

For two candidate states $\theta$ and $\vartheta$, we
measure the ability of action $A$ to distinguish them by
the KL divergence between their Gaussian working
observation models,
\[
d_A(\theta,\vartheta)
=
\frac{
\left(
\theta_{j_A}-\vartheta_{j_A}
\right)^2
}{
2\sigma_A^2
}.
\]
Hence, an action is more informative when the two states
make more widely separated predictions for the estimand it
measures, relative to its anticipated estimation noise.
If $\theta_{j_A}=\vartheta_{j_A}$, then
$d_A(\theta,\vartheta)=0$.

Restricting the maximin design in the main text to the
currently available actions, we compute
\[
\rho_{t+1}
\in
\arg\max_{\rho\in\Delta(\mathcal A_t)}
\min_{(\theta,\vartheta)\in\mathcal P_t}
\sum_{A\in\mathcal A_t}
\rho(A)\,
d_A(\theta,\vartheta).
\]
This allocation maximizes the expected separation of the
least distinguishable frontier pair. We then select up to
two distinct actions according to $\rho_{t+1}$, fix their
statistical specifications, and execute them on the next
previously unseen data fold.


\paragraph{Adaptive hypothesis generation.}
At each expansion step, an LLM proposes a new falsifiable
hypothesis and its logical complement using the research
question, previously generated hypotheses, and all completed
experiments. Proposals are restricted to claims supported by
the experimental interface and are validated before any new
data fold is observed.

For each accepted hypothesis, previously collected evidence
whose prespecified statistical claim applies to that
hypothesis is reused, allowing newly generated hypotheses to
be evaluated against the full accumulated evidence record.
Falsified hypotheses are discarded, while nonfalsified
hypotheses are retained for further experimentation.
Generated hypotheses that become certified enter
$\mathcal C_t$ and are eligible for selection according to
the abductive utility $u(H)$.

\paragraph{Background knowledge and abductive selection.}
We use an LLM to elicit background possibility values~\citep{yang2026verbalizing}
for the negative, zero, and positive states of each estimand,
using background scientific knowledge without exposing the
current run's experimental results. For each estimand, the
three values are normalized by their maximum. For a joint
state $\theta=(\theta_1,\ldots,\theta_J)$, we then define
\[
\pi^K(\theta)
=
\min_{j=1,\ldots,J}\pi_j^K(\theta_j),
\]
yielding a fixed background possibility contour over the
finite working state space.

For each hypothesis $H$, we compute its background utility
using Eq.~\ref{eq:guaranteed-background}. Among the certified
hypotheses $H\in\mathcal C_t$, we select those maximizing $u(H)$,
breaking utility ties in favor of smaller cardinality as in
Eq.~\ref{eq:specificity-tiebreak}.

\paragraph{Hypothesis Matching Score (HMS).}
We evaluate agreement with the DiscoveryBench gold
hypothesis using the Hypothesis Matching Score
(HMS)~\citep{majumder2025discoverybench}, using the
version adopted in AstaBench~\citep{bragg2026astabench}.
HMS measures the alignment between a predicted and gold
hypothesis along three dimensions: their context,
associated variables, and the relationship among those
variables. Each dimension is scored by an LLM judge, and
the final HMS is the product of the three alignment
scores:
\[
\mathrm{HMS}(\hat H,H^\star)
=
s_{\mathrm{ctx}}\,
s_{\mathrm{var}}\,
s_{\mathrm{rel}}.
\]
The score is on a $[0,1]$ scale, with higher
values indicating closer agreement with the gold
hypothesis. The gold hypothesis is used only for this
post-hoc evaluation and is not exposed during hypothesis
generation, experiment selection, or hypothesis
selection.

\subsection{Validity under Adaptive Hypothesis Selection}
\label{app:validity}

This experiment evaluates the evidential criteria defining the certified set $\mathcal{C}_t$ in Eq.~\eqref{eq:evidence-admissibility}, which requires $\underline{\Pi}_t^E(H)>\delta$ and $\underline{N}_t^E(H)\geq 1-\delta$. Specifically, it examines anytime validity and statistical power under post-hoc hypothesis selection, as well as the role of necessity in distinguishing compatibility from sufficient evidential support.

\paragraph{Setup.}
Sequential observations are generated as
$M_t\sim\mathcal N(\theta^\star,1)$ for $t=1,\ldots,100$,
with the fixed candidate family
\[
\mathcal H
=
\left\{
\{\theta>c\},\{\theta<c\}:
c\in\{-1,-0.9,\ldots,1\}
\right\}.
\]
At each round, all $42$ candidates are evaluated using
the same accumulated observations, allowing post-hoc
hypothesis selection. The fixed family facilitates comparison
with Bonferroni correction, although the simultaneous validity
guarantee of our method also accommodates adaptively generated
hypotheses. The certification threshold is set to $\delta=0.05$.

The Gaussian normal-mixture $e$-process
\citep{howard2021time} is given by
\[
E_t(\theta)
=
\frac{1}{\sqrt{1+t\tau^2}}
\exp\left\{
\frac{\tau^2
\left(\sum_{s=1}^t(M_s-\theta)\right)^2}
{2(1+t\tau^2)}
\right\},
\]
which mixes likelihood ratios over alternatives
$\theta+\xi$, where $\xi\sim\mathcal N(0,\tau^2)$.
The resulting $e$-possibility and necessity measures
are computed according to Eq.~\eqref{eq:epos_theta},
with their persistent versions used for certification.

\paragraph{Baselines.}
Two baselines are considered, both employing
hypothesis-specific one-sided Gaussian mixture
$e$-processes. Let
$S_t(c)=\sum_{s=1}^t(M_s-c)$. The corresponding
processes are
\[
E_t^\pm(c)
=
\frac{2}{\sqrt{1+t\tau^2}}
\exp\left\{
\frac{\tau^2 S_t(c)^2}{2(1+t\tau^2)}
\right\}
\Phi\left(
\frac{\pm\tau S_t(c)}{\sqrt{1+t\tau^2}}
\right),
\]
where $E_t^+(c)$ tests $\theta\le c$ and
$E_t^-(c)$ tests $\theta\ge c$.

Both baselines select the candidate with the largest
corresponding $e$-value. \emph{Naive reuse} retains
the selected candidate when its $e$-value exceeds
$1/\delta$, without accounting for post-hoc selection.
\emph{Bonferroni}~\citep{dunn1961multiple} instead
applies the threshold
$|\mathcal H|/\delta=42/\delta$, ensuring
family-wise error control over the fixed candidate
family. No additional temporal correction is needed
because the individual $e$-processes are anytime-valid.

\paragraph{Metrics.}
Following Corollary~\ref{cor:simultaneous-hypothesis},
two family-wise error probabilities are evaluated:
false rejection of a true hypothesis and false
necessity assignment to a false hypothesis,
respectively:
\begin{align*}
&\Pr_{\theta^\star}\left(
\exists t\le T,H\in\mathcal H:
\theta^\star\in H,\,
\underline{\Pi}_t^E(H)\le\delta
\right),\\
&\Pr_{\theta^\star}\left(
\exists t\le T,H\in\mathcal H:
\theta^\star\notin H,\,
\underline{N}_t^E(H)\ge1-\delta
\right).
\end{align*}
For the baselines, the corresponding family-wise
probability of supporting any false candidate is
reported.

Statistical power is measured for the directional
target $H^\star=\{\theta>0\}$ when $\theta^\star>0$
and $H^\star=\{\theta<0\}$ when $\theta^\star<0$,
as the probability of retaining the target by
$T=100$. Power is undefined at $\theta^\star=0$.

To isolate the contribution of necessity, the
following events are compared at $T=100$:
\[
X_1=\{\underline{\Pi}_T^E(H)>\delta\},
\qquad
X_2=\{\underline{\Pi}_T^E(H)>\delta,\,
\underline{N}_T^E(H)\ge1-\delta\}.
\]
Their probabilities are plotted against the
oriented boundary distance
$d=c-\theta^\star$ for $H=\{\theta>c\}$
and $d=\theta^\star-c$ for $H=\{\theta<c\}$,
where $d<0$ indicates that the hypothesis
contains the true parameter.

\begin{figure*}[t]
    \centering
    \begin{subfigure}[b]{0.32\textwidth}
        \centering
        \includegraphics[width=\textwidth]{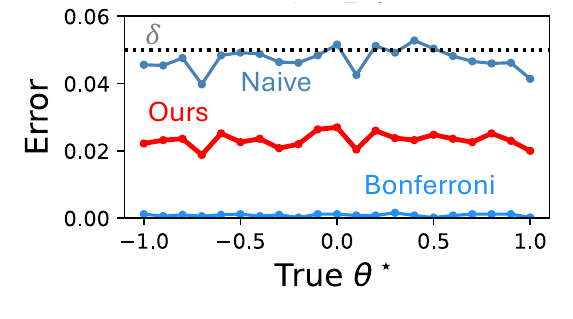}
        \caption{Error}
        \label{fig:exp1_error}
    \end{subfigure}
    \hfill
    \begin{subfigure}[b]{0.32\textwidth}
        \centering
        \includegraphics[width=\textwidth]{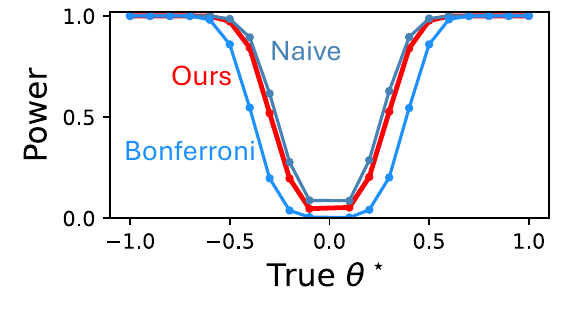}
        \caption{Power}
        \label{fig:exp1_power}
    \end{subfigure}
    \hfill
    \begin{subfigure}[b]{0.32\textwidth}
        \centering
        \includegraphics[width=\textwidth]{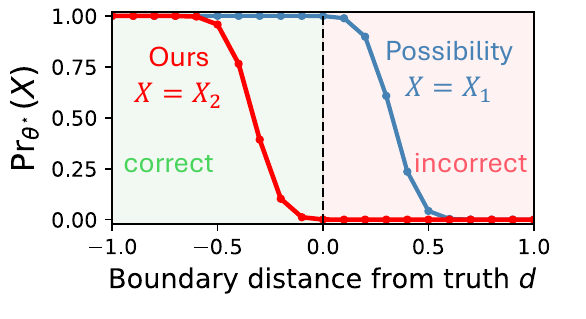}
        \caption{Necessity ablation}
        \label{fig:exp1_ablation}
    \end{subfigure}
    \caption{
    Validity under post-hoc hypothesis selection.
    The certification criteria control both error
    types while retaining competitive power.
    The necessity requirement distinguishes
    compatibility from sufficient evidential support.
    }
    \label{fig:exp1_rst}
\end{figure*}

\paragraph{Results.}
As shown in Fig.~\ref{fig:exp1_error}, the certification criteria keep both error types below the nominal level $\delta=0.05$. Naive reuse slightly exceeds the nominal
level for some states, whereas Bonferroni is
substantially more conservative. The certification
criteria achieve power comparable to naive reuse (Fig. \ref{fig:exp1_power})
while outperforming Bonferroni for moderate
effect sizes (Fig.~\ref{fig:exp1_power}).

The necessity ablation (Fig.~\ref{fig:exp1_ablation})
shows that compatibility can persist near the true
boundary, whereas necessity requires sufficient
evidence against the complement. Smaller values
of $\delta$ impose a stricter necessity requirement,
shifting the transition toward more negative
boundary distances. The primary experiment uses $\tau=1.0$.
Fig. \ref{fig:tau_ablation} shows ablation with
$\tau\in\{0.5,2.0\}$ examine sensitivity
to the Gaussian mixture scale.

\begin{figure}[h]
    \centering
    \includegraphics[width=0.8\linewidth]{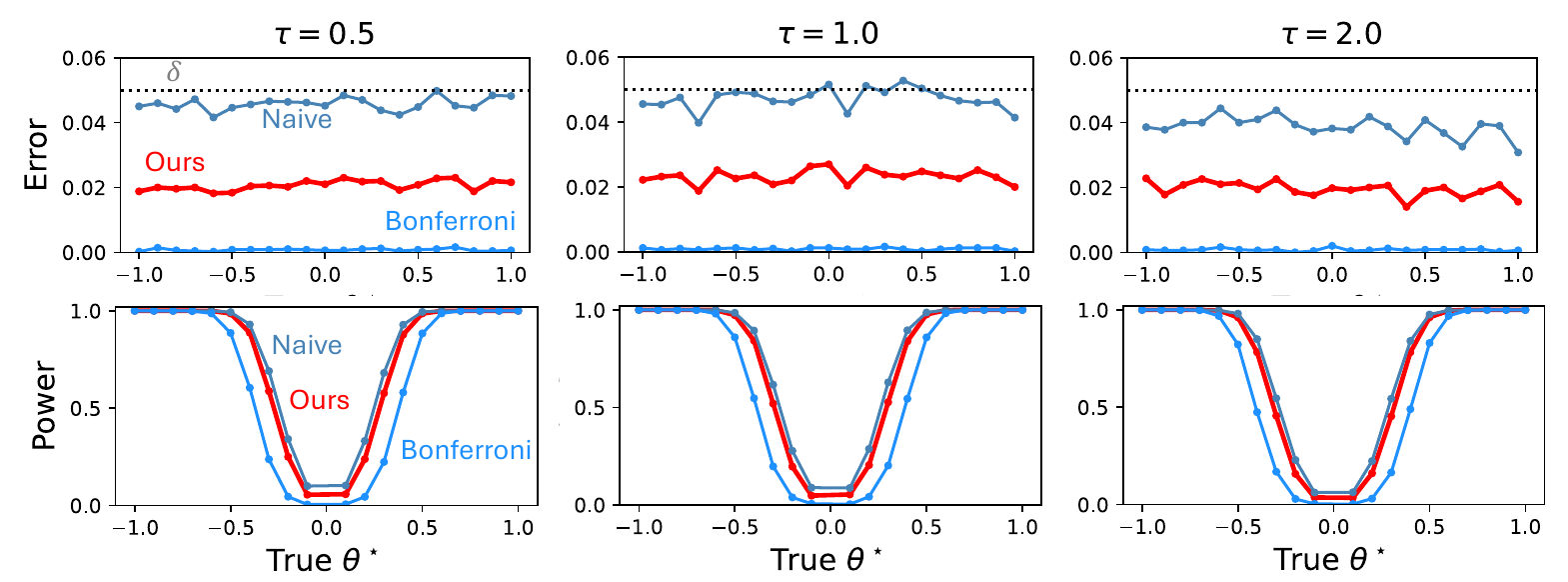}
    \caption{Empirical error (top) and power (bottom) under post-hoc hypothesis selection for different mixture scales $\tau$. Our method maintains error control at $\delta=0.05$ with power comparable to naive reuse across varying $\tau$.}
    \label{fig:tau_ablation}
\end{figure}



\end{document}